\documentclass{article}

\PassOptionsToPackage{authoryear,round}{natbib}
\usepackage{cvzod}

\usepackage[utf8]{inputenc}
\usepackage[T1]{fontenc}

\usepackage{url}
\usepackage{booktabs}
\usepackage{graphicx}
\usepackage{microtype}

\usepackage[dvipsnames]{xcolor}
\definecolor{DarkRed}{rgb}{0.368,0.097,0.078}
\definecolor{DarkBlue}{rgb}{0.2,0.2,0.6}

\usepackage[hypertexnames = false,
    colorlinks = true,
    linkcolor = DarkBlue,
    anchorcolor = DarkBlue,
    citecolor = DarkBlue,
    filecolor = DarkBlue,
    urlcolor = DarkRed]{hyperref}

\usepackage{amsthm} %
\usepackage{mathtools,thmtools}
\usepackage{amsfonts,amsmath,amssymb}
\usepackage[capitalise]{cleveref}
\usepackage{times}
\usepackage[ruled]{algorithm2e}
\usepackage{thm-restate}
\usepackage{tcolorbox}
\usepackage{lipsum}
\usepackage{enumitem}
\usepackage{bm}
\usepackage{xspace}
\usepackage{nicefrac}
\usepackage{dsfont}
\usepackage{float}

\usepackage{bbm}
\usepackage{mathtools,thmtools}

\declaretheoremstyle[
	    spaceabove=\topsep, 
	    spacebelow=\topsep, 
	    headfont=\normalfont\bfseries,
	    bodyfont=\normalfont\itshape,
	    notefont=\normalfont\bfseries,
	    notebraces={(}{)},
	    postheadspace=0.5em, 
	    headpunct={},
	    postfoothook=\ignorespaces
    ]{theorem}
\declaretheorem[style=theorem,numberwithin=section]{theorem}

\declaretheoremstyle[
	    spaceabove=\topsep, 
	    spacebelow=\topsep, 
	    headfont=\normalfont\bfseries,
	    bodyfont=\normalfont,
	    notefont=\normalfont\bfseries,
	    notebraces={(}{)},
	    postheadspace=0.5em, 
	    headpunct={},
	    postfoothook=\ignorespaces
    ]{definition}

\declaretheoremstyle[
        spaceabove=\topsep, 
        spacebelow=\topsep, 
        headfont=\normalfont\bfseries,
        bodyfont=\normalfont,
        notefont=\normalfont\bfseries,
        notebraces={}{},
        postheadspace=0.5em, 
        qed=$\blacksquare$, 
        headpunct={},
        postfoothook=\ignorespaces
    ]{proofstyle}
\declaretheorem[style=proofstyle,numbered=no,name=Proof]{proof}

\declaretheorem[style=theorem,sibling=theorem,name=Lemma]{lemma}
\declaretheorem[style=theorem,sibling=theorem,name=Corollary]{corollary}

\declaretheorem[style=theorem,sibling=theorem,name=Proposition]{proposition}
\declaretheorem[style=theorem,sibling=theorem,name=Fact]{fact}

\declaretheorem[style=theorem,numbered=no,name=Theorem]{theorem*}
\declaretheorem[style=theorem,numbered=no,name=Lemma]{lemma*}
\declaretheorem[style=theorem,numbered=no,name=Corollary]{corollary*}
\declaretheorem[style=theorem,numbered=no,name=Proposition]{proposition*}
\declaretheorem[style=theorem,numbered=no,name=Claim]{claim*}
\declaretheorem[style=theorem,numbered=no,name=Fact]{fact*}
\declaretheorem[style=theorem,numbered=no,name=Observation]{observation*}
\declaretheorem[style=theorem,numbered=no,name=Conjecture]{conjecture*}

\declaretheorem[style=definition,sibling=theorem,name=Remark]{remark}

\declaretheorem[style=definition,numbered=no,name=Definition]{definition*}
\declaretheorem[style=definition,numbered=no,name=Remark]{remark*}
\declaretheorem[style=definition,numbered=no,name=Example]{example*}
\declaretheorem[style=definition,numbered=no,name=Question]{question*}

\DeclareMathOperator{\E}{\mathbb{E}}

\DeclareMathOperator{\tr}{tr}

\newcommand{\ceil}[1]{\lceil #1 \rceil}

\newcommand{\ind}[1]{\mathbb{I}\set{#1}}

\newcommand{\cA}{\mathcal{A}}
\newcommand{\cB}{\mathcal{B}}

\newcommand{\cE}{\mathcal{E}}
\newcommand{\cF}{\mathcal{F}}
\newcommand{\cG}{\mathcal{G}}
\newcommand{\cH}{\mathcal{H}}

\newcommand{\cN}{\mathcal{N}}

\newcommand{\cS}{\mathcal{S}}

\let\papernewcommand\newcommand
\let\newcommand\providecommand

\newcommand{\iid}{\overset{\mathrm{iid}}{\sim}}

\newcommand{\rbr}[1]{\left(#1\right)}
\newcommand{\sbr}[1]{\left[#1\right]}
\newcommand{\cbr}[1]{\left\{#1\right\}}

\newcommand{\abr}[1]{\left|#1\right|}

\newcommand{\pair}[1]{\langle{#1}\rangle} %

\newcommand{\E}{\mathbb{E}}
\renewcommand{\P}{\mathbb P}
\newcommand{\ind}{\mathbb I}

\newcommand{\tnorm}[1]{\ensuremath{\lVert #1 \rVert}}

\newcommand{\ceil}[1]{\left\lceil\, {#1}\,\right\rceil}

\newcommand{\tsum}{\textstyle\sum}

\newcommand\N{\mathbb N}
\newcommand\R{\mathbb R}

\newlength{\pgmtab}  %
 {
	\begin{enumerate}}{\end{enumerate}}

\def\qedsketch{\ifmmode\Box\else{\unskip\nobreak\hfil
\penalty50\hskip1em\null\nobreak\hfil$\Box$
\parfillskip=0pt\finalhyphendemerits=0\endgraf}\fi}

\newlength{\tpush}
\newcommand{\handout}[5]{
   
   \begin{center}
   \framebox{ \vbox{ \hbox to \textwidth { {\bf \coursenum\ :\  \coursename} \hfill #5 }
       \vspace{3mm}
       \hbox to \textwidth { {\Large \hfill #2  \hfill} }
       \vspace{1mm}
       \hbox to \textwidth { {\it #3 \hfill #4} }
     }
   }
   \end{center}
   \vspace*{4mm}
   \newcommand{\lecturenum}{#1}
   \addcontentsline{toc}{chapter}{Lecture #1 -- #2}
}

\newcommand{\normop}[1]{{\left\vert\kern-0.25ex\left\vert\kern-0.25ex\left\vert #1 
		\right\vert\kern-0.25ex\right\vert\kern-0.25ex\right\vert}}

\newcommand{\cA}{\mathcal{A}}
\newcommand{\cB}{\mathcal{B}}

\newcommand{\cE}{\mathcal{E}}
\newcommand{\cF}{\mathcal{F}}
\newcommand{\cG}{\mathcal{G}}
\newcommand{\cH}{\mathcal{H}}

\newcommand{\cN}{\mathcal{N}}

\newcommand{\cS}{{\mathcal{S}}}

\newcommand{\hty}{\widehat{y}}

\newcommand{\myle}[1]{\stackrel{\text{(#1)}}{\le}}

\newcommand{\tr}{\mathrm{tr}}

\let\newcommand\papernewcommand

\newcommand{\smp}{\tau}

\newcommand{\widebar}{\overline}
\newcommand{\spn}{\mathrm{span}}

\newcommand{\Ges}{G}
\newcommand{\Gges}{G}
\newcommand{\Gr}{\mathrm{Gr}}

\renewcommand{\ind}{\mathbb I}

\newcommand{\restatetheorem}[2]{
\begingroup
\def\thetheorem{\ref{#1}}
\begin{theorem}[Restated]
#2
\end{theorem}
\addtocounter{theorem}{-1}
\endgroup
}

\newcommand{\restatelemma}[2]{
\begingroup
\renewcommand{\thelemma}{\ref{#1}}
\begin{lemma}[Restated]
#2
\end{lemma}
\addtocounter{theorem}{-1}
\endgroup
}

\title{Adaptively Incorporating Directional Hints\\ into Zeroth-Order Optimization}

\author{%
  Alexander Ryabchenko \\
  University of Toronto, Vector Institute \\
  \texttt{alex.rbch.research@gmail.com}
  \And
  Jian Qian \\
  University of Hong Kong \\
  \texttt{jianqian@hku.hk}
  \AND
  Wenlong Mou \\
  University of Toronto, Vector Institute \\
  \texttt{wenlong.mou@utoronto.ca}
}
\date{}

\begin{document}
\maketitle

\begin{abstract}
\noindent We study zeroth-order optimization of nonconvex functions with the aid of directional hints, which are cheap but potentially inaccurate approximations of the true gradient direction, given by linear subspaces at each iteration. To leverage these hints adaptively while maintaining robustness to their quality, we introduce Control-Variate Zeroth-Order Descent (CV-ZOD), a new framework that refines the classical zeroth-order gradient estimator with a control variate that can be set based on the directional hints. We first show that the oracle algorithm that optimally sets the reference vector and step size at each iteration achieves a convergence rate that interpolates between the first-order $O(1/T)$ rate and the zeroth-order $O(d/T)$ rate, depending on the quality of the hints along the trajectory. We then develop a practical variant of CV-ZOD that achieves the same oracle guarantee up to logarithmic factors, without any prior knowledge of the hint quality. We validate the method empirically on simulation-based scientific optimization tasks, demonstrating sustained progress on nonconvex landscapes where zeroth-order descent is slower and existing guided methods stall as guidance deteriorates.
\end{abstract}

\section{Introduction}

Many real-world optimization problems involve nonconvex objectives whose gradients are unavailable or prohibitively expensive to compute. In scientific applications, the objective may be defined by a simulator or experimental pipeline that cannot be differentiated end to end \citep{larson2019derivative}. Such settings arise across science and engineering, including, but not limited to, molecular property optimization \citep{hoffman2022optimizing}, autonomous photocatalyst discovery \citep{burger2020mobile}, and battery charging protocol optimization \citep{attia2020closed}. Related challenges arise in machine learning, including black-box adversarial attacks on deployed classifiers \citep{chen2017zoo,ilyas2018black}, prompt optimization for API-gated language models \citep{zhan2024unlocking}, and memory-efficient fine-tuning of large models, where backpropagation is replaced by forward-pass-only updates \citep{malladi2023finetuning}. Zeroth-order methods, which rely only on function evaluations, are a natural tool in these settings. However, their convergence rates typically scale with the dimension $d$, which can make them impractical in high-dimensional problems.

At the same time, auxiliary information about the objective is often available. A pretrained surrogate, a smaller related model, or a low-rank approximation may suggest promising descent directions. For example, the gradient of a distilled model can provide a cheap proxy for the gradient of a large target model in prompt optimization, while a differentiable foundation model can guide black-box optimization in scientific experiments. Using such information effectively, however, is delicate. Blindly following the suggested directions can lead to rapid early progress when they are well aligned with the true gradient, but the resulting bias may ultimately limit convergence. Conversely, ignoring this information altogether reverts to the standard dimension-dependent cost of zeroth-order optimization. Ideally, we would like to use directional hints when they are informative and adaptively fall back to unbiased zeroth-order descent when they are not, without any prior knowledge of their quality along the trajectory.

The idea of incorporating prior directional information into zeroth-order optimization has been explored in several works \citep{maheswaranathan2019guided, cheng2021convergence}. These methods bias the search distribution toward the suggested directions, which can accelerate convergence when the guidance is well aligned with the true gradient. However, this modification also biases the resulting gradient estimators: they no longer target the gradient of the original objective, but rather a reweighted version of it. Consequently, the optimization dynamics are distorted by the guidance, which can limit progress. Moreover, existing convergence analyses for these approaches are restricted to convex objectives. This motivates the central question:
\begin{center}
\emph{Can we incorporate directional hints into nonconvex zeroth-order optimization in a way that adapts to their quality along the trajectory without altering the underlying optimization dynamics?}
\end{center}

In this paper, we answer this question affirmatively. We study a general interaction model in which the hints are revealed as low-dimensional linear subspaces at each iteration, allowing different sources of directional information to be treated within a common framework. Our contributions are as follows:
\begin{enumerate}[before=\vspace{-0.15cm},after=\vspace{-0.15cm},leftmargin=0.7cm,itemsep=0.05cm]
\item \textbf{Control-Variate Zeroth-Order Descent (CV-ZOD) framework (Section~\ref{sec:gradient-estimation}).} 
We introduce a new zeroth-order descent scheme, CV-ZOD, built on a control-variate gradient estimator \citep{greensmith2004variance, owen2013monte} that uses a reference vector to reduce variance while remaining unbiased for the smoothed gradient. The variance of the estimator is governed by the accuracy of the reference vector relative to the true gradient. By instantiating the reference vector through projection of the true gradient onto the hint subspace and choosing the optimal stepsize, the resulting oracle algorithm interpolates between the first-order $O(1/T)$ rate when the hints are well-aligned and the zeroth-order $O(d/T)$ rate when they are uninformative.
\item \textbf{A fully zeroth-order adaptive instantiation (Section~\ref{sec:results}).}
In practice, the oracle information of projected gradient and optimal stepsize is unavailable. We develop a practical variant of CV-ZOD that estimates these quantities using only function evaluations, and show that it achieves the same adaptive convergence guarantee, up to logarithmic factors, with $O(k+\log T)$ queries per iteration. As a result, the adaptive version of CV-ZOD can achieve the interpolating convergence rates without any prior knowledge of the hint quality along the trajectory.
\item \textbf{Simulation-based scientific optimization tasks (Section~\ref{sec:simulations}).}
We validate the framework on scientific optimization tasks arising from fluid dynamics and computational chemistry, using cheap surrogate models to generate directional hints. We show that CV-ZOD achieves significant speedup over existing zeroth-order methods, and much lower final error than existing guided methods that rely on biased gradient estimation.
\end{enumerate}

\paragraph{Related work.} 

Prior theoretical work has explored incorporating surrogate gradient information into zeroth-order optimization, but the existing analyses have so far been limited to convex objectives and relied on biased gradient estimators that reshape the optimization dynamics around the surrogate; the resulting algorithms do not follow the gradient flow of $f$, but one shaped by the surrogate. Specifically, \citet{maheswaranathan2019guided} introduce Guided Evolutionary Strategies (GES), which bias the search distribution along the guiding subspace and analyze the bias-variance tradeoff of the resulting estimator, but provide no convergence guarantees. \citet{cheng2021convergence} analyze the Prior-Guided Random Gradient-Free (PRGF) method, which uses a similar biased gradient estimator, and obtain convergence rates that improve with the cosine similarity between the prior and the true gradient; however, their analysis relies essentially on convexity and a bounded level set assumption, precluding a direct extension to nonconvex objectives. See Appendix~\ref{app:related-work} for a more detailed discussion.

A complementary approach in scientific optimization is Bayesian optimization, which selects experiments through a probabilistic model of objective values. Applications include reaction optimization \citep{shields2021bayesian} and the discovery of improved photocatalyst formulations \citep{burger2020mobile}. CV-ZOD instead uses directional hints to reduce the variance of a zeroth-order gradient estimator while retaining its smoothed-gradient expectation. The black-box machine-learning applications discussed above provide further settings in which cheap surrogate gradients may supply such hints.

\section{Problem setting}\label{sec:setting}

This paper considers the unconstrained minimization problem
\begin{align*}
    f^* := \min_{x \in \R^d} f(x),
\end{align*}
where $d \in \N$ is the dimension and $f\colon \R^d \to \R$ is the differentiable objective. Throughout, $f$ is taken to be $B$-Lipschitz, i.e., $\|\nabla f(x)\| \le B$ for all $x \in \R^d$, 
and $L$-smooth, i.e.,
\begin{align*}
    \bigl|f(y)-f(x)-\pair{\nabla f(x),\,y-x}\bigr|
    \le \tfrac{L}{2}\tnorm{y-x}^2 \qquad \text{for all }x,y \in \R^d.
\end{align*}
We study iterative algorithms that access $f$ only through a \emph{zeroth-order oracle}, which returns the scalar value $f(x)$ for any query $x \in \R^d$, and that, at each iteration, are provided with a \emph{directional hint} in the form of a low-dimensional linear subspace of $\R^d$. Figure~\ref{fig:setting} depicts the resulting interaction.

Formally, the algorithm is initialized at an arbitrary point $x_1 \in \R^d$ and runs for $T \in \N$ iterations. At iteration $t \in [T]$, the environment reveals a hint subspace $\cS_t \in \Gr_k(\R^d)$, where $\Gr_k(\R^d)$ denotes the set of $k$-dimensional linear subspaces of $\R^d$. The algorithm then issues a batch of zeroth-order queries and selects the next iterate $x_{t+1}$. After $T$ iterations, it returns an output point $\bar x \in \R^d$ based on the information collected during the interaction.
For the descent schemes considered below, every step size $\alpha_t$ is strictly positive.

\begin{figure}[H]
\centering
\fbox{\parbox{0.84\columnwidth}{
    \textbf{Zeroth-Order Optimization with Directional Hints}
    \vspace{0.1cm}

    \textit{Parameters}: horizon $T$, hint subspace dimension $k\in [d]$.\\
    \textit{Input}: initial point $x_1 \in \R^d$.

    \vspace{0.15cm}
    For each iteration $t = 1,2,\ldots,T$:
    \begin{enumerate}[leftmargin=0.75cm,noitemsep,before=\vspace{-0.25cm},after=\vspace{-0.125cm}]
        \item The environment reveals a hint subspace $\cS_t \in \Gr_{k}(\R^d)$.
        \item The algorithm makes queries to the zeroth-order oracle and selects the next iterate $x_{t+1}$.
    \end{enumerate}

    \textit{Output}: point $\bar x \in \R^d$ selected using the information collected over the $T$ iterations. 
}}
\caption{Zeroth-order optimization with directional hints.}
\label{fig:setting}
\end{figure}

Performance is measured by the \emph{stationarity of the output}: specifically, we seek to control $\E[\tnorm{\nabla f(\bar x)}^2]$. Here, the expectation is with respect to all randomness that can affect the output point $\bar x$: the algorithm's internal randomization and the hint-generation process that generates hints $(\cS_t)_{t=1}^T$, as specified in the next paragraph.

\paragraph{Hint generation model.}
We impose no distributional or structural assumptions on the mechanism that produces hint subspaces. The sequence $(\cS_t)_{t=1}^T$ may be deterministic or random, and may be chosen adaptively, subject only to a \emph{nonanticipation} condition. Formally, let $(\cF_t)_{t=1}^T$ be a filtration such that $\cF_t$ captures all randomness available before iteration $t$, including the iterate $x_t$. We require only that $\cS_t$ be $\cF_t$-measurable. Equivalently, the hint revealed at iteration $t$ may depend on the current iterate and all past information, but not on the algorithm's internal randomness within the same iteration. For example, given access to differentiable surrogates $f'_1, \ldots, f'_k$, our framework accommodates the natural hint choice $\cS_t = \spn\{\nabla f'_1(x_t), \ldots, \nabla f'_k(x_t)\}$.

We quantify the quality of the hint $\cS_t$ via the \emph{principal angle} between $\cS_t$ and the current gradient:
\begin{align*}
    \theta_t
    :=
    \angle(\nabla f(x_t), \cS_t)
    =
    \arccos\rbr{\tfrac{\tnorm{P_{\cS_t}\nabla f(x_t)}}{\tnorm{\nabla f(x_t)}}},
\end{align*}
with the convention $\theta_t := 0$ when $\nabla f(x_t) = 0$. The angle $\theta_t$ measures how much of the current gradient is captured by $\cS_t$. The sequence $(\theta_t)_{t=1}^T$ is itself a stochastic process adapted to $(\cF_t)_{t=1}^T$, and our convergence guarantees are stated in terms of its realized distribution along the trajectory rather than any worst-case bound.

\paragraph{Additional notation.}
We write $[n] := \{1, \ldots, n\}$ for $n \in \N$. For a subspace $\cS \subseteq \R^d$, we denote by $P_{\cS}$ the orthogonal projector onto $\cS$ and by $\cS^\perp$ its orthogonal complement, with the shorthand $P_{\cS}^{\perp} := P_{\cS^{\perp}}$. For vectors $v_1, \ldots, v_m \in \R^d$, we write $\spn\{v_1, \ldots, v_m\}$ for their linear span. For $k \in \{0, \ldots, d\}$, we denote by $\Gr_k(\R^d)$ the Grassmannian of $k$-dimensional linear subspaces of $\R^d$, with $\Gr_0(\R^d)$ consisting of the trivial subspace $\{0\}$ alone.

\section{Control-variate framework for zeroth-order descent}\label{sec:gradient-estimation}

This section introduces the Control-Variate Zeroth-Order Descent (CV-ZOD) framework. Section~\ref{sec:cv-estimator} constructs a gradient estimator \eqref{eq:cv-estimator} that, given an arbitrary reference vector, refines it into an unbiased estimate of the smoothed gradient with variance controlled by how close the reference vector is to the true gradient. Section~\ref{sec:cv-zod} embeds this estimator into an iterative descent scheme in which a fresh reference vector and step size are chosen at each iteration, and establishes a convergence guarantee (Theorem~\ref{thm:main-theorem}) governed by the interplay between the two along the trajectory. The framework is agnostic to how reference vectors and step sizes are chosen; Section~\ref{sec:oracle-choices} specializes it to directional hints and identifies the oracle choices that serve as a benchmark for Section~\ref{sec:results}, where we match this benchmark using only function evaluations.

\subsection{Gradient estimation with control variates}\label{sec:cv-estimator}

We recall standard zeroth-order tools \citep{ghadimi2013stochastic, nesterov2017random, balasubramanian2019} before introducing our control-variate estimator.

\paragraph{Preliminaries: Gaussian smoothing and the classical gradient estimator.}

For a smoothing parameter $\smp > 0$, we define the $\smp$-smoothed objective by
\begin{align*}
    f^\smp(x)
    :=
    \E_{u \sim \cN(0, I_d)}\sbr{f(x + \smp u)}
    \qquad \text{for all }x \in \R^d.
\end{align*}

The following lemma records the basic approximation properties of Gaussian smoothing from \citet[Section~2]{nesterov2017random}.
\begin{lemma}
\label{lem:gaussian-smoothing}
The function $f^\smp$ is differentiable, $L$-smooth, and, for all $x \in \R^d$, satisfies
\begin{align*}
    |f^\smp(x)-f(x)|
    &\le \frac{\smp^2}{2} Ld \qquad \text{and} \qquad
    \tnorm{\nabla f^\smp(x)-\nabla f(x)}
    \le \frac{\smp}{2} L(d+3)^{3/2}.
\end{align*}
Moreover, for all $x \in \R^d$, $\nabla f^\smp(x)
    =
    \E_{u \sim \cN(0, I_d)}
    \big[\tfrac{f(x + \smp u)}{\smp}\, u\big].$
\end{lemma}

By taking $\smp$ sufficiently small, the smoothed objective $f^\smp$ and its gradient $\nabla f^\smp$ can be made arbitrarily close to $f$ and $\nabla f$, respectively. The final identity expresses $\nabla f^\smp(x)$ through function evaluations alone and motivates the classical (two-point) \emph{zeroth-order gradient estimator}:
\begin{align}\label{eq:G_classical}
    \Ges(x; u)
    :=
    \frac{f(x + \smp u) - f(x)}{\smp}\, u, \qquad \text{where }
    u \sim \cN(0, I_d),
\end{align}
whose properties are summarized below, following \citet[Theorem~3.1]{ghadimi2013stochastic}.
\begin{lemma}
\label{lem:classical-estimator}
For all $x \in \R^d$, 
$\E_{u \sim \cN(0, I_d)}[\Ges(x; u)] = \nabla f^\smp(x)$ and
\begin{align*}
    \E_{u \sim \cN(0, I_d)}\sbr{\tnorm{\Ges(x; u) - \nabla f(x)}^2}
    \le
    4(d+5)\tnorm{\nabla f(x)}^2 + \smp^2 L^2 (d+6)^3.
\end{align*}
\end{lemma}
The $d$-factor in the first term of the mean-squared error bound reflects the cost of isotropic exploration: \eqref{eq:G_classical} explores all directions equally, without any mechanism to prioritize the gradient direction.

\paragraph{The control-variate gradient estimator.}

For every $x \in \R^d$, define the control-variate gradient estimator parameterized by a reference vector $m \in \R^d$ as
\begin{align}\label{eq:cv-estimator}
    \Gges(x,m;u) := m + \frac{f(x+\smp u)-f(x)-\pair{m,\smp u}}{\smp}\,u, \qquad \text{where } u\sim\cN(0,I_d).
\end{align}
The construction applies the classical zeroth-order estimator to $f(\cdot)-\pair{m,\cdot}$, estimating the residual $\nabla f(x) - m$ and adding $m$ back. Equivalently, rearranging gives $\Gges(x,m;u)=\Ges(x;u) + (I_d-uu^\top)m$. The term $(I_d-uu^\top)m$ is a zero-mean control variate: it leaves the bias unchanged while reducing variance when $m$ approximates $\nabla f(x)$. The next lemma makes this precise.

\begin{lemma}\label{lem:cv-estimator}
For all $x,m\in\R^d$, it holds that $\E_{u\sim\cN(0,I_d)}\sbr{\Gges(x,m;u)} = \nabla f^\smp(x)$.
Moreover,
\begin{align*}
    \E_{u\sim\cN(0,I_d)}\sbr{\tnorm{\Gges(x,m;u)-\nabla f(x)}^2}
    \le
    2(d+1)\, \tnorm{m - \nabla f(x)}^2 + 8\smp^2 L^2 d^3
\end{align*}
\end{lemma}
When $m$ closely approximates $\nabla f(x)$, the leading term $d \|m - \nabla f(x)\|^2$ becomes small, substantially improving over Lemma~\ref{lem:classical-estimator} in which the corresponding term scales as $d\,\|\nabla f(x)\|^2$.

\subsection{Descent scheme and general convergence guarantee}\label{sec:cv-zod}

We now embed the control-variate estimator into an iterative descent scheme (Algorithm~\ref{alg:main-algorithm}). At each iteration~$t$, the algorithm selects a reference vector $m_t$ and a step size $\alpha_t$, draws a fresh exploration direction $u_t\sim\cN(0, I_d)$, and takes a gradient step using the control-variate estimator $\Gges(x_t, m_t; u_t)$. The final output is sampled from the iterates with probabilities proportional to the step sizes taken. 

{}

\begin{algorithm}[H]
\caption{Control-Variate Zeroth-Order Descent (CV-ZOD)}\label{alg:main-algorithm}
\SetKwInput{Input}{Input}
\SetKwInput{Output}{Output}
\Input{Initial point $x_1 \in \R^d$.}
\vspace{0.12cm}
For each iteration $t = 1, 2, \ldots, T$:
\begin{enumerate}[before=\vspace{0.05cm},noitemsep,after=\vspace{-0.15cm}, leftmargin=0.75cm]
    \item Select the reference vector $m_t$ and step-size $\alpha_t > 0$.
    \item Update $x_{t+1} = x_t - \alpha_t \Gges(x_t, m_t; u_t)$ for independently sampled $u_t \sim \cN(0, I_d)$.
\end{enumerate}
\Output{$\bar x = x_R$, where $\Pr(R=t\mid\text{realized run}) = \alpha_t/\sum_{s=1}^T\alpha_s$ for $t\in[T]$.}

\vspace{-0.0cm}
\end{algorithm}

The exploration direction $u_t$ is sampled from $\cN(0,I_d)$ independently of all information available when it is drawn, including $(x_t,m_t,\alpha_t)$. The reference vector and step size may depend arbitrarily on the history, including the current iterate and any external available information.

The convergence of Algorithm~\ref{alg:main-algorithm} is governed by the interplay between the step sizes $\alpha_t$ and the quality of the reference vectors $m_t$. By Lemma~\ref{lem:cv-estimator}, taking a step of size $\alpha_t$ incurs estimation variance proportional to $\alpha_t^2 d\,\|m_t - \nabla f(x_t)\|^2$ up to smoothing. Larger step sizes amplify this variance, while more accurate reference vectors reduce it. To quantify whether the step sizes are justified by the reference quality along the trajectory, we introduce the \emph{variance-descent balance}
\begin{align}\label{eq:descent-balance}
    \cB_T^{\gamma}
    :=
    \tsum_{t=1}^T
    \rbr{
        \alpha_t^2 d\,\tnorm{m_t-\nabla f(x_t)}^2
        -
        \tfrac{\alpha_t}{\gamma L}\tnorm{\nabla f(x_t)}^2
    },
\end{align}
which accumulates, over iterations, the variance cost of each step minus a descent credit, with the scale parameter $\gamma \ge 1$ controlling the relative weight between the two. We fix $\gamma = C_0 \log(2T^2)$ for most of our applications, where $C_0$ is the absolute constant from the following theorem. This theorem shows that, for a $B$-Lipschitz and $L$-smooth objective $f$, the sign and magnitude of $\cB_T^{\gamma}$ govern the convergence rate of Algorithm~\ref{alg:main-algorithm}.

\begin{theorem}[Convergence of CV-ZOD]\label{thm:main-theorem}
    Suppose Algorithm~\ref{alg:main-algorithm} uses a fresh standard Gaussian direction $u_t$, conditionally on all information available before its draw, for every $t \in [T]$, and that $\max_{t \in [T]} \tnorm{m_t} \le 2B$ almost surely. Then there exist absolute constants $C_0, C > 0$ such that the following holds. For any $\delta \in (0,1)$, set $\iota := \log(2T/\delta)$, and suppose $d \ge \iota$ and $0 < \alpha_t \le \tfrac{1}{C_0\, L\, \iota}$ for all $t \in [T]$. Then, for $\gamma = C_0 \iota$, with probability at least $1 - \delta$,
    \begin{align*}
        \tfrac{1}{L}\tsum_{t=1}^{T} \alpha_t \tnorm{\nabla f(x_t)}^2
        \;\le\; C\,\rbr{D_f^2 + \tfrac{B^2}{L^2} + \smp^2\, T\, d^3 + \cB_T^{\gamma}\,\iota},
    \end{align*}
    where $D_f^2 := \tfrac{2}{L}\rbr{f(x_1) - f^*}$. Moreover, if $d \ge \log(2T^2)$, $\gamma = C_0 \log(2T^2)$, $0<\alpha_t\le 1/(\gamma L)$ for all $t \in [T]$, and $\smp \le \tfrac{B/L}{\sqrt{Td^3}}$, then $\widebar{x}$ satisfies
    \begin{align*}
        \tfrac{1}{L}\,\E\!\sbr{\tnorm{\nabla f(\widebar{x})}^2}
        \;\le\; C\,\E\!\sbr{\tfrac{D_f^2 + B^2/L^2 + \cB^{\gamma}_T \log(2T^2)}{\cA_T}},
    \end{align*}
    where $\cA_T := \tsum_{t=1}^{T} \alpha_t$ denotes the cumulative step size.
\end{theorem}

The proof, which appears in Appendix~\ref{app:CV-ZOD}, is based on careful control of self-normalized martingales arising from the interaction between the step sizes and the control-variate estimation error. The condition $d \ge \iota = \log(2T/\delta)$ is mild in the high-dimensional regime we consider; when it fails, one may replace $d$ by $d + \iota$ in all bounds. The bound interpolates between two familiar regimes. When the reference vectors are uninformative ($m_t = 0$ for all~$t$), taking $\alpha_t = 1/(\gamma Ld)$ ensures $\cB_T^{\gamma} \le 0$ and recovers the standard zeroth-order rate $O(d/T)$ up to the factor $\gamma$, which is logarithmic for our choice $\gamma = C_0\log(2T^2)$. When the reference vectors match the gradient ($m_t = \nabla f(x_t)$ for all $t$), taking $\alpha_t = 1/(\gamma L)$ again gives $\cB_T^{\gamma} \le 0$ and recovers the first-order rate $O(1/T)$ up to the logarithmic factor. In general, the quality of the reference vectors determines how large the step sizes can be while keeping $\cB_T^{\gamma}$ controlled, and thus the degree of acceleration beyond the zeroth-order baseline.

\begin{remark}[The $B^2/L^2$ term]
The quantities $D_f^2 = \tfrac{2}{L}(f(x_1) - f^*)$ and $B^2/L^2$ are not generally comparable. The former is the standard initial suboptimality and satisfies $D_f^2 \ge \tnorm{\nabla f(x_1)}^2/L^2$. The latter arises from the high-probability martingale concentration used in the proof.
\end{remark}

\subsection{Locally optimal choices from directional hints}\label{sec:oracle-choices}

Theorem~\ref{thm:main-theorem} holds for arbitrary reference vectors and step sizes. To understand the best achievable rate given the hint subspaces $(\cS_t)_{t=1}^T$, consider the idealized choice $m_t = P_{\cS_t}\nabla f(x_t)$, which minimizes $\|m - \nabla f(x_t)\|^2$ over $m \in \cS_t$. Under this choice, the residual satisfies $\|m_t - \nabla f(x_t)\|^2 = \sin^2\theta_t\,\|\nabla f(x_t)\|^2$, and each summand of $\cB^{\gamma}_T$ is nonpositive whenever $\alpha_t \le \tfrac{1}{\gamma Ld\, \sin^2\theta_t}$.

This motivates defining, for any realization of the trajectory, the \emph{oracle} reference vector and step size:
\begin{align}\label{eq:oracle-stepsize}
    m_t^* := P_{\cS_t}\nabla f(x_t), \qquad
    \alpha_t^* := \tfrac{1}{L}\,\min\!\cbr{1,\,\tfrac{1}{d\,\sin^2\theta_t}} \in \sbr{\tfrac{1}{Ld},\,\tfrac{1}{L}}.
\end{align}
They serve as a natural benchmark: $\alpha_t^*/\gamma$ is the largest step size at most $1/(\gamma L)$ for which the $t$-th summand of $\cB^{\gamma}_T$ remains nonpositive under the best reference vector from $\cS_t$. An algorithm using $(m_t^*, \alpha_t^*/\gamma)$ at each iteration would ensure $\cB^{\gamma}_T \le 0$, and Theorem~\ref{thm:main-theorem} would yield
\begin{align}\label{eq:oracle-rate}
    \tfrac{1}{L}\,\E\!\sbr{\tnorm{\nabla f(\widebar x)}^2}
        \le C\,\E\sbr{\tfrac{\gamma\,(D_f^2 + B^2/L^2)}{\cA^*_T}},
\end{align}
where $\cA_T^* := \sum_{t=1}^T \alpha_t^*$. This rate interpolates with the hint quality: when $\theta_t = 0$ for all $t$, we have $\alpha_t^* = 1/L$ and recover ${O}(1/T)$; when $\theta_t = \pi/2$, we have $\alpha_t^* = 1/(Ld)$ and recover ${O}(d/T)$.

These oracle choices depend only on the current pair $(x_t, \cS_t)$:
the reference vector $m_t^*$ minimizes the squared residual over $\cS_t$,
and $\alpha_t^*/\gamma$ is the largest step satisfying the nonpositive-balance
condition and the step-size cap $\alpha_t \le 1/(\gamma L)$.
They depend on the unknown gradient; Section~\ref{sec:results} shows how
to approximate them from function evaluations.

\section{Main results: adaptive convergence from directional hints}\label{sec:results}

This section constructs a fully zeroth-order instantiation of CV-ZOD. Section~\ref{sec:primitives} estimates the projected gradient and the squared norms of the full gradient and its component orthogonal to the hint subspace. Section~\ref{sec:learning-rate} uses these estimates to choose the reference vector and an adaptive step size. Theorem~\ref{thm:adaptive-rate} establishes a convergence bound governed by the cumulative oracle step size $\cA_T^*$ evaluated along the algorithm's realized trajectory. Proofs appear in Appendix~\ref{app:adaptive-results}.

\subsection{Zeroth-order estimation primitives}\label{sec:primitives}

\paragraph{Estimating the in-subspace gradient component.}
To construct the reference vector, we estimate the projected gradient $P_{\cS}\nabla f(x)$ by querying the function once along each vector in an orthonormal basis of $\cS$. This is a finite-difference approximation in subspace coordinates \citep{berahas2022theoretical,chen2026relationships}. Algorithm~\ref{alg:projection} therefore uses $k+1$ oracle queries, and its approximation guarantee is deterministic under $L$-smoothness.

\begin{algorithm}[H]
\caption{$P^{\smp}_{f}(x, \cS)$}\label{alg:projection}
\SetKwInOut{Input}{Input}
\Input{ Point $x \in \R^d$, subspace $\cS$ of dimension $k \in [d]$.}
\vspace{0.1cm}
\begin{enumerate}[before=\vspace{0.0cm},noitemsep,after=\vspace{-0.15cm}, leftmargin=0.75cm]
    \item Construct $B = [b_1,\ldots,b_k] \in \R^{d\times k}$ with orthonormal columns spanning $\cS$.
    \item For $j \in [k]$, compute $\hty_j = \frac{f(x+\smp b_j) - f(x)}{\smp}$.
    \item Output $B\hty$, where $\hty = (\hty_1,\ldots,\hty_k)^\top$.
\end{enumerate}
\end{algorithm}

\begin{lemma}[Deterministic projection estimation]\label{lem:projection}
For every $x \in \R^d$ and $\cS \in \Gr_k(\R^d)$, Algorithm~\ref{alg:projection} satisfies
\begin{align*}
    \|P^{\smp}_{f}(x, \cS) - P_{\cS}\nabla f(x)\|^2
    \le \frac{kL^2\smp^2}{4}.
\end{align*}
\end{lemma}
\vspace{-0.3cm}
The bound follows directly by applying the smoothness remainder to each basis direction; in particular, no concentration argument or logarithmic oversampling is needed.

\paragraph{Estimating gradient norms.}
To set the step size adaptively, we estimate $\|\nabla f(x_t)\|^2$ and $\|P_{\cS_t^\perp}\nabla f(x_t)\|^2$. Both quantities are squared norms of gradient projections, which can be approximated via a common primitive: Algorithm~\ref{alg:norm-est} computes a constant-factor approximation by averaging estimates from $N$ queries.

\begin{algorithm}[H]
\caption{$N^{\smp}_f(x, \cS, N)$}\label{alg:norm-est}
\SetKwInOut{Input}{Input}
\Input{ Point $x \in \R^d$, linear subspace $\cS$ of $\R^d$, number of sampled directions $N$.}
\vspace{0.1cm}
\begin{enumerate}[before=\vspace{-0.1cm},noitemsep,after=\vspace{-0.15cm}, leftmargin=0.75cm]
    \item For $i \in [N]$, sample $u_i \iid \cN(0, I_d)$ and compute $\hty_i = \frac{f(x+\smp P_{\cS} u_i) - f(x)}{\smp}$.
    \item Output $\tfrac{1}{N} \sum_{i=1}^N \hty_i^2$.
\end{enumerate}
\end{algorithm}

\begin{lemma}[Norm estimation]\label{lem:NE-concentration}
There exist absolute constants $C, c > 0$ such that for every $x \in \R^d$, linear subspace $\cS$ of $\R^d$, $\smp > 0$, $\delta \in (0,1)$, and $N \ge C \log(2/\delta)$, it holds with probability at least $1-\delta$ that
\begin{align*}
    N^{\smp}_f(x, \cS, N) \in \sbr{\tfrac{1}{2}\|P_{\cS} \nabla f(x)\|^2 - c\varepsilon,\, \tfrac{3}2\|P_{\cS} \nabla f(x)\|^2 + c\varepsilon},
\end{align*}
where $\varepsilon = \smp^2 L^2 (d^2 + \log^2(2N/\delta))$.
\end{lemma}

The constant-factor accuracy suffices: the adaptive rule below uses these estimates only to identify the correct scale of $\alpha_t^*$ up to a constant factor.

\subsection{Adaptive guarantees}\label{sec:learning-rate}

We now combine the two estimation primitives. The reference vector $m_t$ is set using Algorithm~\ref{alg:projection}. The next lemma defines the adaptive step size $\alpha_t$ via the norm estimates from Algorithm~\ref{alg:norm-est} and establishes two properties: it tracks the oracle step size $\alpha_t^*$ up to a constant factor, and it keeps the descent balance controlled.

\begin{lemma}\label{lem:adaptive-tuning}
There exist absolute constants $C, c_0, c_1 > 0$ such that the following holds. For any $\delta\in(0,1)$ and $N \ge C\log(2T/\delta)$, set $\varepsilon := \smp^2 L^2(d^2 + \log^2(2NT/\delta))$ and define
\begin{align}\label{eq:adaptive-stepsize}
    \alpha_t
    :=
    \tfrac{1}{\gamma\,L}\,\min\cbr{1,\,\tfrac{N^{\smp}_f(x_t,\R^d,N) + c_1\varepsilon}{16d\,\max\{0,\, N^{\smp}_f(x_t,\cS_t^{\perp},N) - c_1\varepsilon\}}},
\end{align}
If the denominator is zero, define $\alpha_t:=1/(\gamma L)$.
Then with probability at least $1-\delta$, for all $t \in [T]$ simultaneously:
    \begin{enumerate}[leftmargin=0.75cm,noitemsep]
        \item (Local Optimality) $\quad \alpha_t \ge \tfrac{\alpha_t^*}{c_0\gamma}$;
        \item (Balance Control) $\quad \alpha_t^2 d\tnorm{P_{\cS_t}^{\perp} \nabla f(x_t)}^2 \le \tfrac{\alpha_t}{2\gamma L}\tnorm{\nabla f(x_t)}^2 + c_0\,\smp^2 (d^3 + d\log^2(2NT/\delta))$.
    \end{enumerate}
\end{lemma}

Local optimality ensures that the realized step sizes are within a constant factor of the locally optimal ones (oracle $\alpha_t^*$); balance control ensures that $\cB^{\gamma}_T$ is dominated by smoothing terms. Substituting both $m_t$ and $\alpha_t$ into Theorem~\ref{thm:main-theorem} yields our main convergence result.

\begin{theorem}\label{thm:adaptive-rate}
There exist absolute constants $C, c > 0$ such that, if $d \ge \log(2T^2)$, Algorithm~\ref{alg:main-algorithm} is run with $m_t = P_f^\smp(x_t,\cS_t)$ and step sizes $\alpha_t$ from \eqref{eq:adaptive-stepsize} with $\gamma = C_0 \log(2T^2)$ and $N = \ceil{C\log(2T)}$, and the smoothing radius satisfies $\smp \le \frac{B/L}{\sqrt{T(d^3 + dk)\log(2T^2)}}$, then
\begin{align*}
    \tfrac{1}{L}\,\E\!\sbr{\tnorm{\nabla f(\widebar{x})}^2}
        \;\le\; c\,\E\!\sbr{\tfrac{(D_f^2 + B^2/L^2)\log(2T)}{\cA^*_T}},
\end{align*}
where $\cA^*_T = \tsum_{t=1}^T \alpha^*_t$ for $\alpha_t^*$ in \eqref{eq:oracle-stepsize}.
\end{theorem}

Thus, the algorithm matches the locally optimal rate up to logarithmic factors. The cumulative oracle step size $\cA_T^*$ governs the rate: the bound interpolates between $O(1/T)$ and $O(d/T)$ depending on how well the hint subspaces align with the gradient along the trajectory. Algorithm~\ref{alg:projection} uses $k+1$ queries, while the norm estimates in Algorithm~\ref{alg:norm-est} use $O(\log T)$ queries. Hence the total per-iteration query cost beyond the single control-variate gradient query is $O(k + \log T)$ --- modest when the hint dimension $k$ is small relative to $d$.

\begin{remark}[Variance reduction by batching]
Our analysis uses a single perturbation $u_t$ per iteration for the control-variate gradient step, with the remaining queries allocated to the estimation primitives. The framework extends directly to a batch of $b$ independent directions $u_t^1,\ldots,u_t^b \iid \cN(0,I_d)$. From Lemmas~\ref{lem:gaussian-smoothing} and~\ref{lem:cv-estimator},
\begin{align*}
    \E\!\big[\|\tfrac{1}{b}\tsum_{i=1}^b G(x_t,m_t;u_t^i)-\nabla f(x_t)\|^2\,\big|\,x_t,m_t\big] \le \tfrac{4d}{b}\|m_t-\nabla f(x_t)\|^2+24\smp^2L^2d^3.
\end{align*}
Thus batching reduces estimator's variance by a factor of $b$, up to Gaussian-smoothing bias. The convergence analysis can be modified with a proportional improvement in the convergence bound.
\end{remark}

\section{Experiments}\label{sec:simulations}

We evaluate CV-ZOD on two simulation-based optimization tasks where a cheap surrogate provides directional hints of varying quality: a fluid inverse problem (Section~\ref{sec:exp-fluid}) and molecular geometry optimization (Section~\ref{sec:exp-molecular}). In both, the optimization updates access the expensive objective $f$ through function evaluations, while the differentiable surrogate $f'$ supplies rank-one hint subspaces $\cS_t = \spn\{\nabla f'(x_t)\}$. We compare against classical zeroth-order descent (ZOD), which uses no hints \citep{ghadimi2013stochastic}, and Guided Evolutionary Strategies (GES), which incorporates the hint subspace through a biased estimator \citep{maheswaranathan2019guided}. All these methods share the same per-iteration query budget. For reference, we also include surrogate gradient descent (SurGD), which takes gradient steps along $\nabla f'(x_t)$, effectively optimizing the surrogate rather than $f$.\footnote{Our code is publicly available at \url{https://github.com/alex-rbch/cv-zod}.}
\vspace{0.2cm}
\begin{remark}[Implementation of the adaptive step size]\label{rmk:adaptive-exp}
In practice, we found it more stable and query-efficient to smooth the norm estimates over time using exponential discounting. Let
$a_t:=N_f^\smp(x_t,\R^d,N_{\mathrm{full}})$ and
$b_t:=N_f^\smp(x_t,\cS_t^\perp,N_{\perp})$, where
$N_{\mathrm{full}}$ and $N_{\perp}$ are the numbers of queries allocated to
the two norm estimates, and fix a discount factor
$\gamma_{\mathrm{disc}}\in[0,1)$. Starting from $A_0=B_0=0$, we update
\begin{align*}
    A_t = \gamma_{\mathrm{disc}} A_{t-1}+a_t,
    \qquad
    B_t = \gamma_{\mathrm{disc}} B_{t-1}+b_t,
    \qquad
    \alpha_t^{\mathrm{exp}} = {A_t}/{B_t}.
\end{align*}
In our simulations, we use this smoothed, recency-biased approximation to the adaptive rule. We support this implementation choice with an ablation experiment in Appendix~\ref{app:simulations}.
\end{remark}

\subsection{Fluid inverse problem}\label{sec:exp-fluid}

\textbf{Fluid dynamics model.}
Recovering initial conditions from terminal observations in fluid systems is a classical PDE-constrained inverse problem with applications in geophysics and weather forecasting \citep{gunzburger2003perspectives}. We study a concrete instance: a passive dye is released into a two-dimensional incompressible fluid governed by the Navier--Stokes equations, discretized on an $n \times n$ spatial grid via the semi-Lagrangian advection--projection scheme of \citet{stam1999stable}. The dye density $\rho_s \in \R^{n \times n}$ evolves jointly with a velocity field $v_s \in \R^{n \times n \times 2}$ over $s = 0,\ldots,T_{\mathrm{sim}}$ time steps. Given a fixed initial density $\rho_0$ and a prescribed target $\rho^\star \in \R^{n \times n}$, the goal is to find an initial velocity field $v_0 \in \R^{n \times n \times 2}$ that minimizes $f(v_0) := \tnorm{\rho_{T_{\mathrm{sim}}}(v_0) - \rho^\star}_{\mathrm F}^{2}$,
where $\rho_{T_{\mathrm{sim}}}(v_0)$ is the terminal density obtained by evolving the dynamics from $(\rho_0, v_0)$. We identify the decision variable $v_0$ with a vector $x \in \R^d$ for $d = 2n^2$ and view $f\colon\R^d\to\R$ accordingly.

\textbf{Oracle and surrogate.}
To simulate the dynamics, we use the differentiable fluid solver provided by NVIDIA Warp \citep{warp2022}. The solver supports reverse-mode automatic differentiation, but backpropagating through the full $T_{\mathrm{sim}}$-step rollout naively requires storing all intermediate density and velocity fields, at $O(T_{\mathrm{sim}} \cdot n^2)$ memory cost. Gradient checkpointing can reduce this to $O(\sqrt{T_{\mathrm{sim}}} \cdot n^2)$ at the cost of additional forward recomputation \citep{griewank2000revolve}, but the memory overhead remains substantial for long time horizons. Forward passes, used by our algorithm as zeroth-order oracle calls, require no intermediate storage.
To construct the directional hints, we use a coarse surrogate $f'\colon\R^d\to\R$ defined by the same solver over a shorter time horizon $T_{\mathrm{sur}} \ll T_{\mathrm{sim}}$. Since $T_{\mathrm{sur}}$ is small, backpropagation through $f'$ is cheap.

\textbf{Experimental setup.}
We set the grid size to $n=64$ ($d=2n^2=8192$), the simulation horizon to $T_{\mathrm{sim}}=64$ steps, and the surrogate horizon to $T_{\mathrm{sur}}=16$ steps. Each zeroth-order method receives a per-iteration budget of $16$ function evaluations. ZOD and GES use all $16$ queries for batching, while CV-ZOD allocates $12$ to the batched control-variate gradient step, $2$ to the gradient projection (Algorithm~\ref{alg:projection}), and $2$ to the norm estimates used by the adaptive step-size rule (Remark~\ref{rmk:adaptive-exp}). Appendix~\ref{app:simulations} reports ablations over the discount factor and query allocation for this fluid problem; we present the best-performing configuration here.

We use $20$ paired seeds and run each trajectory for $100$ optimization steps. All runs start from a low-variance Gaussian perturbation of zero initial velocity to avoid the nondifferentiable stationary point. The initial density $\rho_0$ and target $\rho^\star$ are a vertical and horizontal rectangle, respectively, centered on the grid (Figure~\ref{fig:outputs}, right).

\begin{figure}[H]
    \centering
    \resizebox{0.8\linewidth}{!}{%
        \includegraphics[width=\linewidth]{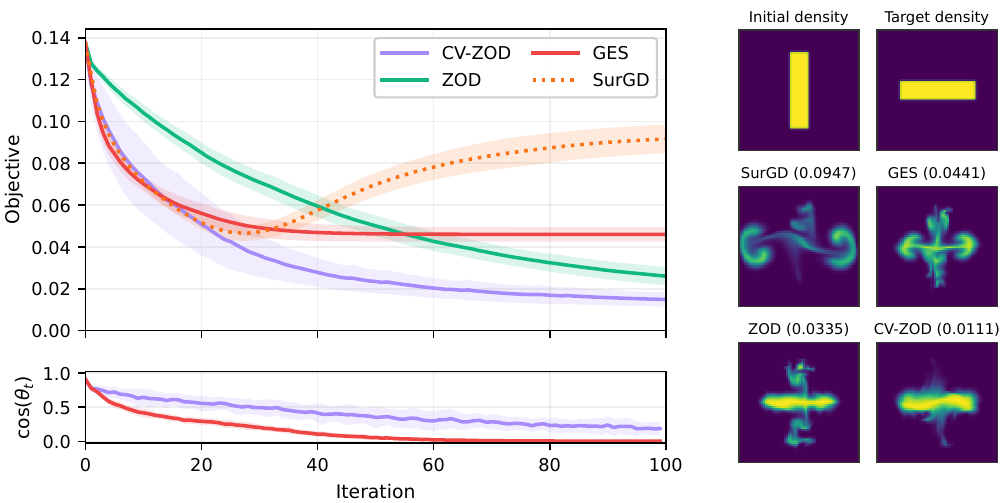}%
    }
    \caption{\textbf{Fluid inverse problem.} Left: objective averaged over $20$ runs (top) and cosine
    similarity between the surrogate and true gradients averaged over the
    same runs (bottom).
    Right: terminal density fields from one paired run produced by
    each method, alongside the initial density and target.
    The shaded bands show mean $\pm$ one sample standard deviation.}
    \label{fig:outputs}
    \label{fig:fluid-terminal}
\end{figure}

\textbf{Results.} Figure~\ref{fig:outputs} (left) shows the objective value averaged over $20$ paired runs, alongside $\cos(\theta_t)$ for CV-ZOD and GES. Figure~\ref{fig:outputs} (right) shows terminal density fields from a common seed. Early in optimization, when the surrogate gradient is well-aligned with the true gradient, CV-ZOD descends at a rate comparable to GES and surrogate descent. As the alignment deteriorates, GES plateaus and SurGD loses much of its early improvement. CV-ZOD, by contrast, transitions smoothly to zeroth-order exploration and reaches the lowest final objective. Additionally, in Appendix~\ref{app:simulations}, we show that expanding the hint subspace with past gradients from the same surrogate does not improve and can even hurt optimization under a fixed query budget.

\subsection{Molecular geometry optimization}\label{sec:exp-molecular}

\textbf{Conformation problem and energy surrogates.}
Finding low-energy molecular conformations is a central task in computational chemistry \citep{hawkins2017conformation}. Given a molecule with $N$ atoms, the goal is to find atomic positions $x \in \R^{3N}$ minimizing the potential energy. In practice, the true energy surface is accessed through expensive quantum-mechanical calculations. For our experiments, we substitute the MACE-OFF23 machine-learned interatomic potential \citep{kovacs2025maceoff} as a realistic proxy, setting $f(x) := E_{\mathrm{MACE}}(x)$ and treating it as a zeroth-order oracle. As surrogate, we use the MMFF94 classical force field \citep{halgren1996merck}, 
an analytic potential with closed-form gradients. Since MACE-OFF23 is itself a differentiable PyTorch model, its gradients are available as a diagnostic, allowing us to compute the true principal angles $\angle(\nabla f(x_t), \cS_t)$ for post-hoc analysis.

\textbf{Experimental setup.}
We use the Wiggle150 benchmark \citep{brew2025wiggle150}, which provides highly strained conformations of efavirenz ($N=30$), adenosine ($N=32$), and benzylpenicillin ($N=41$), with $d=3N$. We use ten starting conformations per molecule. Each trajectory runs for $100$ optimization steps, with starting coordinates and seeds shared across the algorithms. We test the same algorithms with the same query allocations as in the fluid experiments (Section~\ref{sec:exp-fluid}).

\textbf{Results.} Figure~\ref{fig:wiggle} shows energy changes and running cosine similarities for each molecule, averaged over $10$ paired runs. In this setting, the surrogate gradient retains useful alignment with the true gradient throughout optimization, and both CV-ZOD and GES outperform unguided ZOD. The SurGD reference steadily reduces the true energy, illustrating the quality of the surrogate. CV-ZOD achieves a similar energy reduction and outperforms GES.

\begin{figure}[H]
    \centering
    \includegraphics[width=\linewidth]{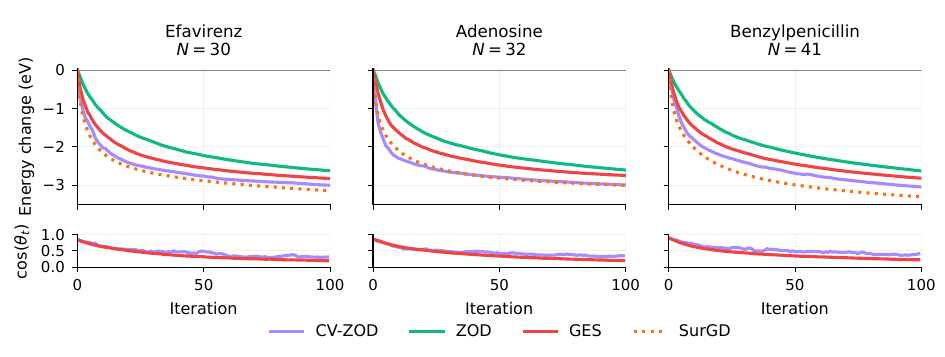}
    \caption{\textbf{Molecular geometry optimization.} MACE energy change from
    initialization (top; lower is better) and cosine similarity between surrogate and true gradients along
    CV-ZOD and GES trajectories (bottom) for efavirenz ($d=90$), adenosine ($d=96$), and benzylpenicillin ($d=123$).}
    \label{fig:wiggle}
\end{figure}

\section{Discussion and future work}

We introduced CV-ZOD, a framework for incorporating directional hints into zeroth-order optimization through control-variate gradient estimation. The key property of the framework is that hints affect only the variance of the gradient estimator, not its bias, so the optimization trajectory tracks the true objective regardless of hint quality. The adaptive step-size rule matches the oracle rate up to logarithmic factors without requiring any prior knowledge of the alignment between the hints and the gradient along the trajectory, at a per-iteration query cost of only $O(k+\log T)$.

Several directions remain open. First, the current framework processes a single low-dimensional hint subspace per iteration. In practice, one may have access to multiple surrogates whose gradients accumulate over iterations; naively taking their span as the hint subspace causes the dimension $k$ to grow, eroding low-dimensionality and increasing the required number of queries. A natural question is whether one can adaptively aggregate a fixed number of surrogate gradients in an optimal way---for instance, through an online learning scheme over the surrogates. Second, the most compelling test of the framework's scalability would be in large-scale settings such as memory-efficient fine-tuning or prompt optimization of language models, where zeroth-order methods are already in use and surrogate gradients from smaller or distilled models arise naturally. Validating CV-ZOD in these high-dimensional regimes is an important direction for future work.

\begin{ack}
This work was partially supported by NSERC grant RGPIN-2024-05092 and a
Connaught New Researcher Award to Wenlong Mou.
\end{ack}

\bibliography{refs}

\clearpage

\appendix

\section{Prior work on zeroth-order optimization with directional hints}\label{app:related-work}

This section analyzes the gradient-estimation techniques used by the two prior approaches for incorporating directional guidance into zeroth-order optimization: Guided Evolutionary Strategies (GES) of \citet{maheswaranathan2019guided} and Prior-Guided Random Gradient-Free (PRGF) methods of \citet{cheng2021convergence}. Both reduce exploration variance along a guided subspace by modifying the search distribution, but in doing so introduce bias into the gradient estimator. We show that this bias creates a fundamental step-size dilemma that the control-variate construction in CV-ZOD avoids.

\paragraph{Guided Evolutionary Strategies.}

The approach of \citet{maheswaranathan2019guided} replaces isotropic perturbations with anisotropic ones whose covariance allocates more variance to directions in the hint subspace $\cS \in \Gr_k(\R^d)$. Concretely, one uses the same estimator $\Ges(x;u)$, but samples perturbations according to
\begin{align*}
    u \sim \cN(0, \Sigma),
    \qquad \text{where }\,\,
    \Sigma := \gamma\, I_d + (1-\gamma)\, \tfrac{d}{k} P_{\cS},
\end{align*}
with $\gamma \in [0,1]$ interpolating between isotropic exploration $(\gamma = 1)$ and exploration supported entirely on $\cS$ $(\gamma = 0)$, while preserving the total variance $\tr(\Sigma) = d$.

The main drawback is that this introduces an anisotropic bias. Defining $f^\smp_\Sigma(x) := \E_{u \sim \cN(0, \Sigma)}[f(x + \smp u)]$, a direct calculation shows that
\begin{align*}
    \E_{u \sim \cN(0, \Sigma)}[\Ges(x; u)]
    =
    \Sigma\, \nabla f^\smp_\Sigma(x)
    \,\,\xrightarrow{\smp \to 0}\,\,
    \Sigma \nabla f(x)
    =
    \gamma\, \nabla f(x) + (1-\gamma)\, \tfrac{d}{k} P_{\cS} \nabla f(x).
\end{align*}
Thus the estimator no longer targets $\nabla f(x)$ itself, but applies different gains to its $\cS$ and $\cS^\perp$ components, amplifying the former by a factor of order $d/k$. This creates a fundamental step-size dilemma: a step size chosen to remain stable along $\cS$ is overly conservative on $\cS^\perp$, while one tuned for $\cS^\perp$ is too aggressive along $\cS$. Anisotropic exploration does not simply reduce variance; it changes the optimization geometry itself, making it ill-suited as a robust mechanism for incorporating directional hints whose alignment with the gradient may vary across iterations.

\paragraph{Prior-Guided Random Gradient-Free methods.}

\citet{cheng2021convergence} propose a related approach in a more restrictive setting. Their guidance takes the form of a single vector $p \in \R^d$, or equivalently the one-dimensional subspace $\cS=\spn\{p\}$, and they assume access to directional derivatives $\pair{\nabla f(x),u}$ rather than zeroth-order feedback $f(x)$. Their estimator with batch size $q \in \N$ can be written as
\begin{align*}
    G_{\mathrm{PRGF}}(x,\cS;\{u_i\}_{i=1}^q)
    =
    P_{\cS}\nabla f(x)
    +
    \sum_{i=1}^q \pair{\nabla f(x),u_i}u_i,
\end{align*}
where $u_1,\ldots,u_q$ follow a Haar-uniform orthonormal $q$-frame in $\cS^\perp$.
This estimator is generally biased:
\begin{align*}
    \E_U\sbr{G_{\mathrm{PRGF}}(x,\cS;U)}
    &=
    P_{\cS}\nabla f(x) + \tfrac{q}{d-1} P^{\perp}_{\cS}\nabla f(x)
    = \tfrac{q}{d-1}\, \nabla f(x) + \rbr{1- \tfrac{q}{d-1}} P_{\cS} \nabla f(x).
\end{align*}
As with GES, the estimator applies different effective gains to the $\cS$ and $\cS^\perp$ gradient components.

\paragraph{Summary.} Both GES and PRGF incorporate directional guidance by modifying the search distribution, which unavoidably biases the gradient estimator. The bias introduces two coupled scales in the update — one along $\cS$ and one along $\cS^\perp$ — that cannot be simultaneously controlled by a single step size. By contrast, the control-variate estimator in CV-ZOD remains unbiased for all reference vectors, so the step size controls a single scale (the residual variance), enabling the adaptive tuning developed in Section~\ref{sec:results}.

\paragraph{Related subspace methods.} Related subspace methods use orthogonal random directions \citep{kozak2023zeroth}, probabilistic direct search in reduced spaces \citep{roberts2023direct}, random embeddings for global optimization \citep{cartis2023global}, and adaptively learned ES-active subspaces \citep{choromanski2019complexity}. A related control-variate construction using surrogate gradients appears in \citet{arisaka2024accelerating}, who apply it to learning to accelerate legacy numerical solvers.

\section{Experiments}\label{app:simulations}

{\raggedright
Code for reproducing the experiments is publicly available at
\url{https://github.com/alex-rbch/cv-zod}.\par}

\subsection{Ablations of the discount factor and query allocation}\label{app:discount-ablation}

Using the fluid setup in Section~\ref{sec:exp-fluid}, we vary
$\gamma_{\mathrm{disc}}\in\{0,0.5,0.6,0.7,0.8,0.9\}$ and compare query
allocations $12/2/2$, $7/7/2$, and $2/12/2$.

\begin{table}[H]
    \centering
    \small
    \caption{Final fluid objective: mean $\pm$ sample standard deviation
    over $20$ paired runs. Columns give query allocations
    (update / norm / projection). Lower is better.}
    \label{tab:discount-ablation}
    \renewcommand{\arraystretch}{1.12}%
    \begin{tabular}{|c|c|c|c|c|}
        \hline
        & \multicolumn{3}{c|}{CV-ZOD} & ZOD \\
        \cline{2-5}
        $\gamma_{\mathrm{disc}}$
        & $12/2/2$ & $7/7/2$ & $2/12/2$ & $16/0/0$ \\
        \hline
        $0$
        & $0.0761 \pm 0.0202$
        & $0.0815 \pm 0.0142$
        & $0.0847 \pm 0.0003$
        & $0.0261 \pm 0.0037$ \\
        $0.5$
        & $0.0165 \pm 0.0100$
        & $0.0163 \pm 0.0043$
        & $0.0846 \pm 0.0003$
        & $0.0261 \pm 0.0037$ \\
        $0.6$
        & $0.0146 \pm 0.0038$
        & $0.0158 \pm 0.0035$
        & $0.0852 \pm 0.0130$
        & $0.0261 \pm 0.0037$ \\
        $0.7$
        & $\mathbf{0.0146} \pm 0.0029$
        & $0.0162 \pm 0.0028$
        & $0.0798 \pm 0.0112$
        & $0.0261 \pm 0.0037$ \\
        $0.8$
        & $0.0149 \pm 0.0035$
        & $0.0162 \pm 0.0042$
        & $0.0778 \pm 0.0201$
        & $0.0261 \pm 0.0037$ \\
        $0.9$
        & $0.0157 \pm 0.0045$
        & $0.0165 \pm 0.0029$
        & $0.0805 \pm 0.0108$
        & $0.0261 \pm 0.0037$ \\
        \hline
    \end{tabular}
\end{table}

Table~\ref{tab:discount-ablation} shows that temporal smoothing has its
largest benefit when enough queries remain for the control-variate
update. With the $12/2/2$ allocation, setting
$\gamma_{\mathrm{disc}}=0.7$ reduces the mean final objective from
$0.07607$ to $0.01464$.
For both $12/2/2$ and $7/7/2$, every displayed positive discount factor
gives a lower mean final objective than ZOD. Increasing the norm-estimation
allocation to $12$ leaves only two update directions and gives much poorer
results throughout the sweep. Among the displayed configurations, the
$12/2/2$ allocation with $\gamma_{\mathrm{disc}}=0.7$ has the lowest
mean final objective. We use this setting for the fluid comparison in
Figure~\ref{fig:outputs}, evaluated on fresh paired runs.

\subsection{Effect of the hint-subspace dimension.}\label{app:hint-dimension}

We vary the number of historical gradients used to construct the hint
subspace while keeping the surrogate function fixed. The objective and
surrogate are the fluid losses from Section~\ref{sec:exp-fluid}, with rollout
horizons $64$ and $16$, respectively. For each history length
$k\in\{1,2,4,8\}$, a run computes one fresh surrogate gradient
$g'_t:=\nabla f'(x_t)$ at iteration $t\in[100]$ and forms
\begin{align*}
    \cS_t=\spn\{g'_t,g'_{t-1},\ldots,g'_{t-\ell_t+1}\},
    \qquad \ell_t:=\min\{k,t\}.
\end{align*}
Algorithm~\ref{alg:projection} estimates the projection of the
\emph{current} true gradient at $x_t$ using one function evaluation
per direction and one shared evaluation of $f(x_t)$, for a total
of $\ell_t+1$ function evaluations.

We reuse the experimental settings of Section~5.1, with a budget of
$16$ queries per iteration for each method. Once $k$ surrogate
gradients are available, CV-ZOD allocates $k+1$ queries to projection
estimation, $2$ to norm estimation, and $13-k$ to the update batch.
The norm queries use one direction in the full space and one in
$\cS_t^\perp$, reusing the evaluation of $f(x_t)$ from projection
estimation. ZOD allocates all $16$ queries to its update batch.
The CV-ZOD update batches therefore contain $12$, $11$, $9$, and $5$
directions for $k=1,2,4,8$, respectively.

\begin{table}[H]
    \centering
    \small
    \caption{Final objective after $100$ steps: mean $\pm$ sample standard
    deviation over $20$ paired runs.}
    \label{tab:hint-dimension}
    \renewcommand{\arraystretch}{1.12}%
    \begin{tabular}{|l|c|}
        \hline
        Method & Final objective \\
        \hline
        ZOD & $0.0260 \pm 0.0044$ \\
        CV-ZOD ($k=1$) & $\mathbf{0.0148} \pm 0.0034$ \\
        CV-ZOD ($k=2$) & $0.0160 \pm 0.0061$ \\
        CV-ZOD ($k=4$) & $0.0170 \pm 0.0029$ \\
        CV-ZOD ($k=8$) & $0.0491 \pm 0.0331$ \\
        \hline
    \end{tabular}
\end{table}

The lowest mean final objective occurs at $k=1$, although the differences
among $k=1,2,4$ are small relative to the between-run standard deviations.
Increasing the history to $k=4$ gives no further mean improvement, and $k=8$
increases both the mean and variability of the final objective. Figure~\ref{fig:hint-dimension} shows a late increase in the $k=8$ mean
objective, reflecting its less reliable performance across runs.

\begin{figure}[H]
    \centering
    \includegraphics[width=0.9\linewidth]{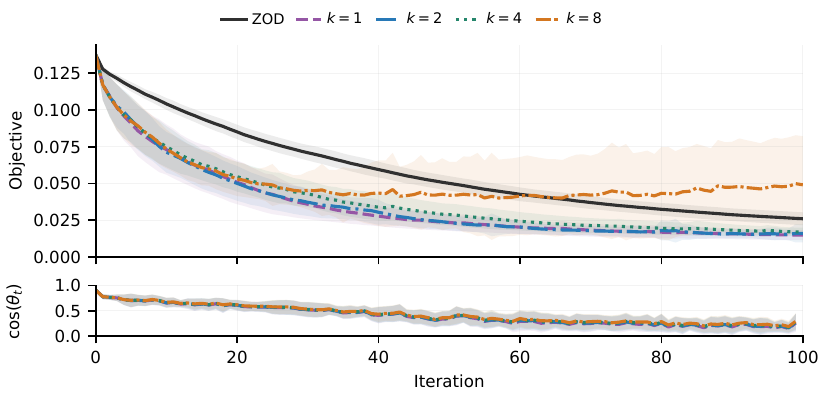}
    \caption{\textbf{Effect of surrogate-gradient history length.}
    Top: objective trajectories under the same query budget.
    Bottom: $\cos(\theta_t)$ using the most recent $k$ surrogate gradients.
    The shaded bands show mean $\pm$ one sample standard deviation.}
    \label{fig:hint-dimension}
\end{figure}

\subsection{Stability of exponential smoothing}
\label{app:exponential-smoothing}

Remark~\ref{rmk:adaptive-exp} replaces fresh norm estimates by exponentially
weighted running sums. The following proposition gives a sufficient condition
under which their ratio tracks the current norm ratio. The result holds along
the realized trajectory and allows the hint subspaces to depend on the
optimization history.

Write $p_t=\|\nabla f(x_t)\|^2$,
$q_t=\|P_{\cS_t^\perp}\nabla f(x_t)\|^2$, and
$\beta=\gamma_{\mathrm{disc}}\in[0,1)$. We retain the definitions of
$a_t,b_t,A_t,B_t$ from Remark~\ref{rmk:adaptive-exp}, with $A_0=B_0=0$.

\begin{proposition}[Tracking under slow norm variation]
\label{prop:exponential-smoothing}
Suppose $p_t,q_t>0$ and that, for some $\eta\in[0,1)$,
the estimates satisfy $|a_t-p_t|\le\eta p_t$ and
$|b_t-q_t|\le\eta q_t$ for all $t\in[T]$.
Assume also that
$|\log(p_{t+1}/p_t)|\le\nu$ and
$|\log(q_{t+1}/q_t)|\le\nu$ for all $t<T$,
where $\nu\ge0$ and $\beta e^\nu<1$.
Set
$K=\frac{1+\eta}{1-\eta}
    \frac{1-\beta e^{-\nu}}{1-\beta e^\nu}$.
Then
$p_t/(Kq_t)\le A_t/B_t\le Kp_t/q_t$
for every $t\in[T]$.

In particular, for any $\gamma\ge1$, the step size
\[
\widehat\alpha_t
=\frac{1}{2K\gamma L}
 \min\left\{1,\frac{A_t}{dB_t}\right\}
\]
satisfies
\[
\frac{\alpha_t^*}{2K^2\gamma}
\le\widehat\alpha_t
\le\frac{\alpha_t^*}{2\gamma},
\qquad
\widehat\alpha_t^2 d q_t
\le\frac{\widehat\alpha_t p_t}{2\gamma L},
\]
where $\alpha_t^*=L^{-1}\min\{1,p_t/(dq_t)\}$ is the
oracle step size from~\eqref{eq:oracle-stepsize}.
\end{proposition}

\begin{proof}
Fix $t\in[T]$ and let $H_\pm=\sum_{j=0}^{t-1}(\beta e^{\pm\nu})^j$.
The variation assumption gives
$e^{-j\nu}p_t\le p_{t-j}\le e^{j\nu}p_t$.
Together with the estimation bounds, this yields
$(1-\eta)p_tH_-\le A_t\le(1+\eta)p_tH_+$,
and the analogous bounds for $B_t$ with $q_t$.
The geometric-series identity gives
$H_+/H_-\le(1-\beta e^{-\nu})/(1-\beta e^\nu)$,
which proves the ratio bounds.

The map $r\mapsto\min\{1,r\}$ preserves multiplicative
comparisons, giving the two-sided bound on $\widehat\alpha_t$.
Finally,
$\widehat\alpha_t\le\alpha_t^*/(2\gamma)
\le p_t/(2\gamma Ldq_t)$
implies the balance inequality, completing the proof.
\end{proof}

When $K$ is bounded by an absolute constant, the conservatively
scaled rule above retains constant-factor tracking of $\alpha_t^*/\gamma$
and balance control. The relative-accuracy and slow-variation assumptions
are additional conditions on the realized trajectory. This conditional
guarantee complements the empirical evaluation of exponential smoothing
in Table~\ref{tab:discount-ablation}.

\section{Proofs for Section~\ref{sec:gradient-estimation}: CV-ZOD framework}\label{app:CV-ZOD}

This section contains the proofs of Lemma~\ref{lem:cv-estimator} and Theorem~\ref{thm:main-theorem}. The proof of Theorem~\ref{thm:main-theorem} relies on two concentration results for self-normalized martingales (Theorem~\ref{thm:mgf-concentration} and Corollary~\ref{cor:peeling}), which adapt the exponential inequalities of \citet{fan2015exponential} to our setting via a peeling argument; these are stated and proved after the main argument.

\subsection{Proof of Lemma~\ref{lem:cv-estimator}}

\restatelemma{lem:cv-estimator}{
    For all $x,m\in\R^d$, it holds that $\E_{u\sim\cN(0,I_d)}\sbr{\Gges(x,m;u)} = \nabla f^\smp(x)$.
    Moreover,
    \begin{align*}
        \E_{u\sim\cN(0,I_d)}\sbr{\tnorm{\Gges(x,m;u)-\nabla f(x)}^2}
        \le
        2(d+1)\, \tnorm{m - \nabla f(x)}^2 + 8\smp^2 L^2 d^3
    \end{align*}
}
\begin{proof}
    Fix $x, m \in \R^d$. The unbiasedness follows immediately from Lemma~\ref{lem:classical-estimator} and Fact~\ref{fact:gaussian-outer-product}:
    \begin{align*}
        \E_{u}\sbr{G(x,m;u)} = \E_{u}\sbr{G(x;u)} + \E_u\!\sbr{(I_d - u u^{\top})m} = \nabla f^\smp(x) + 0.
    \end{align*}
    To bound the mean-squared estimation error, we write
    \begin{align*}
        G(x,m;u) - \nabla f(x)
        &= \rbr{I_d - uu^{\top}}\rbr{m - \nabla f(x)}
        + \frac{f(x+\smp u) - f(x) - \pair{\nabla f(x), \smp u}}{\smp}\, u.
    \end{align*}
    Consequently, we have
    \begin{align*}
        \E_u\!\sbr{\tnorm{G(x,m;u) - \nabla f(x)}^2}
        &\le 2(m - \nabla f(x))^{\top} \E_u\! \sbr{(I_d - uu^{\top})^{2}} (m - \nabla f(x))\\
        &+ 2\E_u\!\sbr{\rbr{\tfrac{f(x+\smp u) - f(x) - \pair{\nabla f(x), \smp u}}{\smp}}^2 \tnorm{u}^2}\\
        &\myle{a} 2(d+1) \tnorm{m - \nabla f(x)}^2 + \tfrac{1}{2}\smp^2 L^2\, \E[\tnorm{u}^6]\\
        &\myle{b} 2(d+1)\, \tnorm{m - \nabla f(x)}^2 + 8\smp^2 L^2 d^3,
    \end{align*}
    where (a) follows from $L$-smoothness and Fact~\ref{fact:gaussian-outer-product} and (b) from Fact~\ref{fact:gaussian-power-norm}.
\end{proof}

\subsection{Proof of Theorem~\ref{thm:main-theorem}}

\restatetheorem{thm:main-theorem}{
Suppose Algorithm~\ref{alg:main-algorithm} uses a fresh standard Gaussian direction $u_t$, conditionally on all information available before its draw, for every $t \in [T]$, and that $\max_{t \in [T]} \tnorm{m_t} \le 2B$ almost surely. Then there exist absolute constants $C_0, C > 0$ such that the following holds. For any $\delta \in (0,1)$, set $\iota := \log(2T/\delta)$, and suppose $d \ge \iota$ and $0 < \alpha_t \le \tfrac{1}{C_0\, L\, \iota}$ for all $t \in [T]$. Then, for $\gamma = C_0 \iota$, with probability at least $1 - \delta$,
    \begin{align*}
        \tfrac{1}{L}\tsum_{t=1}^{T} \alpha_t \tnorm{\nabla f(x_t)}^2
        \;\le\; C\,\rbr{D_f^2 + \tfrac{B^2}{L^2} + \smp^2\, T\, d^3 + \cB_T^{\gamma}\,\iota},
    \end{align*}
    where $D_f^2 := \tfrac{2}{L}\rbr{f(x_1) - f^*}$. Moreover, if $d \ge \log(2T^2)$, $\gamma = C_0 \log(2T^2)$, $0<\alpha_t\le 1/(\gamma L)$ for all $t \in [T]$, and $\smp \le \tfrac{B/L}{\sqrt{Td^3}}$, then $\widebar{x}$ satisfies
    \begin{align*}
        \tfrac{1}{L}\,\E\!\sbr{\tnorm{\nabla f(\widebar{x})}^2}
        \;\le\; C\,\E\!\sbr{\tfrac{D_f^2 + B^2/L^2 + \cB^{\gamma}_T \log(2T^2)}{\cA_T}},
    \end{align*}
    where $\cA_T := \tsum_{t=1}^{T} \alpha_t$ denotes the cumulative step size.
}

\begin{proof}
For each $t \in [T]$, let $\cF_t$ contain all information available before iteration $t$, including the current iterate and hint, and let $\cF_{T+1}$ contain all information after the final iteration. Let $\cG_t$ additionally contain all information available when $\alpha_t$ and $m_t$ have been chosen, just before drawing $u_t$. These sigma-algebras are nested as $\cF_t\subseteq\cG_t\subseteq\cF_{t+1}$.
Throughout, set $g_t := \nabla f^\smp(x_t)$, $w_t := m_t - \nabla f(x_t)$, $\beta_t := \alpha_t - L\alpha_t^2$, $G_t := \Gges(x_t,m_t;u_t)$, and $\Delta_t := G_t - g_t$. Thus $x_t$, $\alpha_t$, and $m_t$ are $\cG_t$-measurable, while $u_t$ is conditionally standard Gaussian given $\cG_t$. Hence Lemma~\ref{lem:cv-estimator} gives $\E[\Delta_t \mid \cG_t] = 0$. Since $\iota \ge \log 2$ and we take $C_0 \ge 4$, the cap $\alpha_t \le \tfrac{1}{C_0 L\iota} \le \tfrac{1}{2L}$ implies $\beta_t \in [\tfrac{1}{2}\alpha_t,\alpha_t]$. We also write
\begin{align*}
    S_T := \tsum_{t=1}^T\alpha_t\tnorm{\nabla f(x_t)}^2,
    \qquad
    V_T := d\tsum_{t=1}^T\alpha_t^2\tnorm{w_t}^2,
    \qquad
    H := D_f^2+B^2/L^2.
\end{align*}

We record two approximation bounds from Lemma~\ref{lem:gaussian-smoothing} used repeatedly:
\begin{align}
    \sup_{x \in \R^d} \tnorm{\nabla f^\smp(x) - \nabla f(x)}
    &\le c_0\, \smp L d^{3/2}, \label{eq:hp-smooth-grad}\\
    f^\smp(x_1) - f^*
    &\le \tfrac{L}{2}D_f^2 + \tfrac{\smp^2 Ld}{2}. \label{eq:hp-smooth-val}
\end{align}

\paragraph{Step 1: Descent inequality.}
By $L$-smoothness of $f^\smp$ and $G_t = g_t + \Delta_t$,
\begin{align*}
    f^{\smp}(x_{t+1})
    &\le f^{\smp}(x_t) - \alpha_t \pair{g_t, G_t} + \tfrac{L\alpha_t^2}{2}\tnorm{G_t}^2\\
    &= f^{\smp}(x_t)
    - \underbrace{(\alpha_t - \tfrac{L\alpha_t^2}{2})}_{\ge\,\alpha_t/2}\,\tnorm{g_t}^2
    - \underbrace{(\alpha_t - L\alpha_t^2)}_{=\,\beta_t}\,\pair{g_t, \Delta_t}
    + \tfrac{L\alpha_t^2}{2}\tnorm{\Delta_t}^2.
\end{align*}
Telescoping over $t = 1, \ldots, T$ and using $f^\smp(x_{T+1}) \ge f^*$ with~\eqref{eq:hp-smooth-val},
\begin{align}\label{eq:hp-descent}
    \tfrac{1}{2}\tsum_{t=1}^T \alpha_t \tnorm{g_t}^2
    \le \tfrac{L}{2}D_f^2 + \tfrac{\smp^2 Ld}{2} - M_T + \tfrac{L}{2}\,Q_T,
\end{align}
where $M_T := \sum_{t=1}^T \beta_t\pair{g_t, \Delta_t}$ and $Q_T := \sum_{t=1}^T \alpha_t^2\tnorm{\Delta_t}^2$.

\paragraph{Step 2: Pathwise bound on $Q_T$.}

The proof of Lemma~\ref{lem:cv-estimator} gives
\begin{align*}
    G_t - \nabla f(x_t) = (I_d - u_tu_t^\top)\,w_t + \epsilon_t\, u_t,
    \qquad
    |\epsilon_t| \le \tfrac{L\smp}{2}\tnorm{u_t}^2.
\end{align*}
Since $\Delta_t = (G_t - \nabla f(x_t)) + (\nabla f(x_t) - g_t)$, we have
\begin{align}\label{eq:hp-Delta-decomp}
    \tnorm{\Delta_t}^2
    \le 2\tnorm{(I_d - u_tu_t^\top)w_t + \epsilon_t u_t}^2 + 2c_0^2\,\smp^2 L^2 d^3.
\end{align}

We establish pathwise control on $u_t$. Since $\tnorm{u_t}^2 \sim \chi_d^2$ conditioned on $\cG_t$, the Laurent--Massart inequality (Lemma~\ref{lem:laurent-massart}) gives $\Pr(\tnorm{u_t}^2 > d + 2\sqrt{d\ell'} + 2\ell') \le e^{-\ell'}$ for all $\ell' > 0$. Similarly, $\pair{w_t, u_t}\mid\cG_t \sim \cN(0, \tnorm{w_t}^2)$, so the standard Gaussian tail yields $\Pr(|\pair{w_t, u_t}| > \tnorm{w_t}\sqrt{2\ell'}) \le 2e^{-\ell'}$. Set $\ell := \log(12T/\delta)\le 4\iota$. A union bound over $t \in [T]$ gives total failure probability at most $3Te^{-\ell}=\delta/4$. Thus the following event $\cE$ holds with probability at least $1-\delta/4$: for all $t \in [T]$ simultaneously,
\begin{align}\label{eq:hp-event}
    \tnorm{u_t}^2 \le 4(d + \ell)
    \qquad \text{and} \qquad
    \pair{w_t, u_t}^2 \le 2\tnorm{w_t}^2\,\ell.
\end{align}
We work on $\cE$ for the remainder of the step.

\emph{Bounding the two components.} Expanding and applying~\eqref{eq:hp-event}:
\begin{align*}
    \tnorm{(I_d - u_tu_t^\top)w_t}^2
    &= \tnorm{w_t}^2 - 2\pair{w_t,u_t}^2 + \pair{w_t,u_t}^2\tnorm{u_t}^2\\
    &\le \tnorm{w_t}^2 + 8\,\tnorm{w_t}^2\,\ell\,(d + \ell)\\
    &\le c\,\tnorm{w_t}^2\, d\,\iota,
\end{align*}
and
\begin{align*}
    \epsilon_t^2\tnorm{u_t}^2
    \le \tfrac{L^2\smp^2}{4}\tnorm{u_t}^6
    \le c\,L^2\smp^2 d^3.
\end{align*}
Substituting into~\eqref{eq:hp-Delta-decomp},
\begin{align}\label{eq:hp-Delta-bound}
    \tnorm{\Delta_t}^2
    \le c\, d\,\iota\,\tnorm{w_t}^2 + c\,\smp^2 L^2 d^3.
\end{align}
Multiplying by $\alpha_t^2$ and summing:
\begin{align*}
    Q_T
    \le c\,d\,\iota\tsum_{t=1}^T \alpha_t^2\tnorm{w_t}^2
    + c\,\smp^2 L^2 d^3 \tsum_{t=1}^T \alpha_t^2.
\end{align*}

Since $\sum_t\alpha_t^2\le T/L^2$, this gives
\begin{align}\label{eq:hp-QT}
    Q_T \le c\,\iota V_T+c\,\smp^2 T d^3.
\end{align}
We retain the nonnegative quantity $V_T$ until the final step, where we use the identity $\iota\cB_T^\gamma=\iota V_T-S_T/(C_0L)$.

\paragraph{Step 3: Martingale concentration.}

\emph{Dominant martingale.} Define
\begin{align*}
    \xi_t := \beta_t\rbr{\pair{\nabla f(x_t), w_t} - \pair{\nabla f(x_t), u_t}\pair{w_t, u_t}}.
\end{align*}
Since $\beta_t, \nabla f(x_t), w_t$ are $\cG_t$-measurable and $\E[\pair{\nabla f(x_t), u_t}\pair{w_t, u_t} \mid \cG_t] = \pair{\nabla f(x_t), w_t}$, each $\xi_t$ is conditionally mean zero. It is measurable with respect to $\cF_{t+1}$, so the partial sums of $\widetilde M_T := \sum_{t=1}^T \xi_t$ form a martingale with respect to the post-draw filtration $(\cG_{t+1})_{t=0}^T$, where $\cG_{T+1}:=\cF_{T+1}$. By Lemma~\ref{lem:Isserlis} with $a = \nabla f(x_t)$ and $\Sigma = w_t w_t^\top$,
\begin{align*}
    \E[\xi_t^2 \mid \cG_t] = \beta_t^2\bigl(\tnorm{\nabla f(x_t)}^2\tnorm{w_t}^2 + \pair{\nabla f(x_t), w_t}^2\bigr) < \infty.
\end{align*}

\emph{MGF condition.} Conditioned on $\cG_t$, $\xi_t$ is a centered degree-two polynomial in $u_t \sim \cN(0, I_d)$. By the Hanson--Wright inequality, there exists an absolute constant $c_1 > 0$ such that
\begin{align*}
    \log \E\!\sbr{e^{\lambda \xi_t} \mid \cG_t} \le c_1^2\lambda^2\beta_t^2\tnorm{\nabla f(x_t)}^2\tnorm{w_t}^2
    \qquad \text{for } |\lambda| \le \tfrac{1}{c_1 \beta_t \tnorm{\nabla f(x_t)}\tnorm{w_t}}.
\end{align*}
Since $e^x \le 1 + 2x$ for $x \in [0,1]$ and $\E[\xi_t^2 \mid \cG_t] \ge \beta_t^2\tnorm{\nabla f(x_t)}^2\tnorm{w_t}^2$ by Lemma~\ref{lem:Isserlis},
\begin{align*}
    \E\!\sbr{e^{\lambda \xi_t} \mid \cG_t}
    \le 1 + 2c_1^2\lambda^2\,\E[\xi_t^2 \mid \cG_t]
    \qquad \text{for } |\lambda| \le \tfrac{1}{c_1 \beta_t \tnorm{\nabla f(x_t)}\tnorm{w_t}}.
\end{align*}
If $\beta_t\tnorm{\nabla f(x_t)}\tnorm{w_t}=0$, then $\xi_t=0$ and the MGF bound holds for every $\lambda$. Otherwise the displayed range applies. The boundedness assumptions give $\tnorm{w_t}\le 3B$ and hence $c_1\beta_t\tnorm{\nabla f(x_t)}\tnorm{w_t} \le 3c_1B^2/(C_0L\iota)$. Set
\begin{align*}
    \varepsilon := \frac{\max\{3c_1,18\}\,B^2}{C_0L\iota}.
\end{align*}
The hypothesis of Theorem~\ref{thm:mgf-concentration} then holds with this $\varepsilon$ and its constant $c_0=4c_1^2$.

\emph{Bounding $\widetilde M_T$.} Using $\beta_t \le \alpha_t$, $\tnorm{w_t} \le 3B$, and $\alpha_t \le 1/(C_0 L\iota)$:
\begin{align*}
    \E[\xi_t^2 \mid \cG_t]
    \le 18\alpha_t^2\tnorm{\nabla f(x_t)}^2 B^2
    \le \tfrac{18B^2}{C_0 L\iota}\cdot \alpha_t\tnorm{\nabla f(x_t)}^2
    \le \varepsilon\cdot \alpha_t\tnorm{\nabla f(x_t)}^2.
\end{align*}
Setting $r_t := \alpha_t\tnorm{\nabla f(x_t)}^2$, Corollary~\ref{cor:peeling} gives, with probability at least $1 - \delta/2$,
\begin{align}\label{eq:hp-tilde-mart}
    |\widetilde M_T| \le \tfrac{1}{8}\tsum_{t=1}^T \alpha_t\tnorm{\nabla f(x_t)}^2 + \tfrac{c\,B^2}{L},
\end{align}
where we used $\varepsilon\log(2e/\delta)\le cB^2/L$, since $\log(2e/\delta)\le c\iota$.

\emph{Smoothing residual.} Since $\Delta_t = (I_d - u_tu_t^\top)w_t + \epsilon_t u_t + (\nabla f(x_t)-g_t)$,
\begin{align*}
    \beta_t\pair{g_t, \Delta_t} - \xi_t
    = \beta_t\pair{g_t - \nabla f(x_t),\, \Delta_t}
    + \beta_t\pair{\nabla f(x_t),\, (\nabla f(x_t)-g_t) + \epsilon_t u_t}.
\end{align*}
On $\cE$, using~\eqref{eq:hp-smooth-grad}, \eqref{eq:hp-Delta-bound}, $|\epsilon_t| \le cL\smp(d+\iota)$, $\tnorm{u_t} \le c\sqrt{d+\iota}$, and $\beta_t \le \alpha_t \le 1/(C_0 L\iota)$:
\begin{align*}
    |\beta_t\pair{g_t, \Delta_t} - \xi_t|
    \le c\,\alpha_t\smp L d^{3/2}\tnorm{\Delta_t}
    + c\,\alpha_t \tnorm{\nabla f(x_t)}\smp L(d + \iota)^{3/2}.
\end{align*}
Applying Young's inequality to the first term with quadratic contribution $L\alpha_t^2\tnorm{\Delta_t}^2/2$ and to the second with contribution $\alpha_t\tnorm{\nabla f(x_t)}^2/32$, then summing and using $\sum_t\alpha_t\le T/L$, gives
\begin{align}\label{eq:hp-residual}
    |M_T - \widetilde M_T|
    \le \tfrac{L}{2}\,Q_T
    + \tfrac{1}{32}S_T+c\,\smp^2 T d^3 L.
\end{align}
Combining~\eqref{eq:hp-tilde-mart} and~\eqref{eq:hp-residual}:
\begin{align}\label{eq:hp-martingale}
    |M_T| \le \tfrac{5}{32}S_T
    + \tfrac{c\,B^2}{L}
    + \tfrac{L}{2}\,Q_T
    + c\,\smp^2 T d^3 L.
\end{align}

\paragraph{Step 4: Combining.}

Intersecting $\cE$ with the martingale event gives an event $\cH$ with probability at least $1-3\delta/4\ge 1-\delta$. On $\cH$, substituting~\eqref{eq:hp-martingale} into~\eqref{eq:hp-descent} and using $-M_T\le |M_T|$ gives
\begin{align*}
    \tfrac{1}{2}\tsum_{t=1}^T \alpha_t\tnorm{g_t}^2
    \le \tfrac{L}{2}D_f^2 + \tfrac{c\,B^2}{L}
    + \tfrac{5}{32}S_T
    + L\,Q_T
    + c\,\smp^2 T d^3 L.
\end{align*}
The gradient approximation~\eqref{eq:hp-smooth-grad} and $\sum_t\alpha_t\le T/L$ imply
\begin{align*}
    \tfrac{1}{2}\tsum_{t=1}^T \alpha_t\tnorm{g_t}^2
    \ge \tfrac{1}{4}S_T-c\,\smp^2 T d^3 L.
\end{align*}
Consequently,
\begin{align*}
    \tfrac{3}{32}S_T
    \le \tfrac{L}{2}D_f^2+\tfrac{cB^2}{L}+LQ_T+c\,\smp^2 T d^3 L.
\end{align*}
Applying~\eqref{eq:hp-QT} and dividing by $L$ therefore gives an absolute constant $K\ge 1$, independent of $C_0$ once $C_0\ge4$, such that
\begin{align}\label{eq:hp-gt-bound}
    \tfrac{S_T}{L}\le K\rbr{H+\smp^2 T d^3+\iota V_T}.
\end{align}

Set $K_2:=2K+1$, and choose $C_0\ge\max\{4,2K_2\}$ and $C:=2K_2$. The balance identity and~\eqref{eq:hp-gt-bound} give, on $\cH$,
\begin{align*}
    H+\smp^2 T d^3+\iota\cB_T^\gamma
    &=H+\smp^2 T d^3+\iota V_T-\tfrac{S_T}{C_0L}\\
    &\ge (1-K/C_0)\rbr{H+\smp^2 T d^3+\iota V_T}\\
    &\ge \tfrac{1}{2}\rbr{H+\smp^2 T d^3+\iota V_T}.
\end{align*}
Combining this inequality with~\eqref{eq:hp-gt-bound} proves the high-probability conclusion with the chosen $C$.

For the expected conclusion, first suppose $T\ge2$ and apply the preceding argument with $\delta=1/T$, so $\iota=\log(2T^2)$. Since the step sizes are positive, $\cA_T>0$. Define
\begin{align*}
    Y:=\frac{S_T}{L\cA_T},
    \qquad X:=\frac{H+\iota V_T}{\cA_T},
    \qquad W:=\frac{H+\iota\cB_T^\gamma}{\cA_T}
             =X-\frac{Y}{C_0}.
\end{align*}
The output rule gives $\E Y=L^{-1}\E\tnorm{\nabla f(\widebar x)}^2$. The smoothing condition implies $\smp^2 T d^3\le B^2/L^2$, so~\eqref{eq:hp-gt-bound} gives $Y\le2KX$ on $\cH$. Moreover, $X\ge0$ and $0\le Y\le B^2/L$ everywhere. Taking expectations while retaining the nonnegative $X$, we obtain
\begin{align*}
    \E Y
    &\le 2K\E X+\frac{B^2}{LT}
     \le K_2\E X.
\end{align*}
The last inequality uses $\cA_T\le T/(C_0\iota L)$, which implies $X\ge B^2/(L^2\cA_T)\ge B^2/(LT)$. If $\E X<\infty$, then
\begin{align*}
    \E W=\E X-\frac{\E Y}{C_0}
    \ge (1-K_2/C_0)\E X\ge\tfrac12\E X.
\end{align*}
Thus $\E Y\le C\E W$, as claimed. If $\E X=\infty$, the same conclusion holds in the extended sense because $Y$ is bounded and $W\ge-B^2/(C_0L)$.

Finally, when $T=1$, the output is $x_1$. With $\iota=\log2$ and the same definitions of $Y,X,W$, positivity and the step-size cap give
\begin{align*}
    W\ge\frac{B^2}{L^2\alpha_1}-\frac{Y}{C_0}
      \ge\rbr{C_0\log2-\frac1{C_0}}\frac{B^2}{L}
      \ge Y.
\end{align*}
This proves the expected conclusion for every $T\in\N$ and completes the proof.
\end{proof}

\begin{theorem}\label{thm:mgf-concentration}
Let $(S_k, \cF_k)_{k=0,\ldots,n}$ be a martingale with $S_0 = 0$ and differences $\xi_i = S_i - S_{i-1}$. Let $c_0 > 0$. Suppose there exists $\varepsilon > 0$ such that for all $i \in [n]$ and all $|\lambda| \le 1/\varepsilon$, 
\begin{align*}
    \E\!\sbr{e^{\lambda \xi_i} \mid \cF_{i-1}} \le 1 + \tfrac{c_0\lambda^2}{2}\,\E\!\sbr{\xi_i^2 \mid \cF_{i-1}}.
\end{align*}
Then for all $x, w > 0$,
\begin{align*}
    \Pr\!\rbr{S_n \ge x \text{ and } \langle S \rangle_n \le w} \le \exp\!\rbr{-\tfrac{x^2}{2(c_0w + \varepsilon x)}},
\end{align*}
where $\langle S \rangle_n = \sum_{i=1}^n \E[\xi_i^2 \mid \cF_{i-1}]$.
\end{theorem}
\begin{proof}
We apply Theorem~2.1 of \citet{fan2015exponential} with $g(\lambda) = 0$, $f(\lambda) = \tfrac{c_0\lambda^2}{2}$, $V_{i-1} = \E[\xi_i^2 \mid \cF_{i-1}]$, and $v \to \infty$. Their bound (2.3) gives, for any $\lambda \in (0, 1/\varepsilon)$,
\begin{align*}
    \Pr\!\rbr{S_n \ge x \text{ and } \langle S \rangle_n \le w}
    \le \exp\!\rbr{-\lambda x + \tfrac{c_0\lambda^2}{2} w}.
\end{align*}
Substituting $\lambda = \frac{x}{c_0 w + \varepsilon x} \in (0, 1/\varepsilon)$:
\begin{align*}
    -\lambda x + \frac{c_0\lambda^2}{2} w
    &= -\frac{x^2}{c_0 w + \varepsilon x} + \frac{x^2}{(c_0 w + \varepsilon x)^2} \cdot \frac{c_0 w}{2}\\
    &\le -\frac{x^2}{c_0 w+\varepsilon x} + \frac{x^2}{2(c_0 w+\varepsilon x)}
    = -\frac{x^2}{2(c_0 w+\varepsilon x)},
\end{align*}
which yields the result.
\end{proof}

\begin{corollary}\label{cor:peeling}
Under the conditions of Theorem~\ref{thm:mgf-concentration}, suppose further that there exist nonnegative $\cF_{i-1}$-measurable random variables $r_1,\ldots,r_n$ such that
\begin{align*}
    \E\!\sbr{\xi_i^2 \mid \cF_{i-1}} \le \varepsilon\, r_i
    \qquad \text{for all } i \in [n] \text{ almost surely}.
\end{align*}
Let $R=\sum_{i=1}^n r_i$. Then for any $\delta\in(0,1)$, with probability at least
$1-\delta$,
\begin{align*}
    |S_n| \le \frac{R}{8} + c\varepsilon\log(e/\delta),
\end{align*}
where $c>0$ depends only on $c_0$.
\end{corollary}

\begin{proof}
It suffices to prove the upper-tail bound, since the lower-tail bound follows by
applying the same argument to $-S_n$. For $j\ge 1$, define
\begin{align*}
    E_j=\{2^{j-1}\varepsilon < R \le 2^j\varepsilon\},
    \qquad
    x_j=\frac{2^{j-1}\varepsilon}{8}+r,
    \qquad
    w_j=\varepsilon^2 2^j,
\end{align*}
and also set $E_0=\{R\le \varepsilon\}$, $x_0=r$, and $w_0=\varepsilon^2$.
On $E_j$, we have $\langle S\rangle_n\le \varepsilon R\le w_j$ and
$R/8+r\ge x_j$. Therefore, by Theorem~\ref{thm:mgf-concentration},
\begin{align*}
    \Pr\!\rbr{S_n \ge R/8+r}
    \le
    \sum_{j\ge 0}
    \exp\!\rbr{-\frac{x_j^2}{2(c_0w_j+\varepsilon x_j)}}.
\end{align*}
We now bound the summands. For $j\ge 1$, since
$x_j\ge 2^j\varepsilon/16$, we have
\begin{align*}
    c_0w_j = c_0\varepsilon^2 2^j
    \le 16c_0\varepsilon x_j.
\end{align*}
Hence, for a constant $C>0$ depending only on $c_0$,
\begin{align*}
    \frac{x_j^2}{2(c_0w_j+\varepsilon x_j)}
    \ge
    \frac{x_j}{C\varepsilon}
    \ge
    \frac{2^{j-1}}{C}+\frac{r}{C\varepsilon}.
\end{align*}
For the term $j=0$, after increasing $C$ if necessary, the same bound gives
\begin{align*}
    \frac{x_0^2}{2(c_0w_0+\varepsilon x_0)}
    =
    \frac{r^2}{2(c_0\varepsilon^2+\varepsilon r)}
    \ge
    \frac{r}{C\varepsilon},
\end{align*}
provided $r\ge c_0\varepsilon$, which will be true for the choice of $r$ below.
Thus
\begin{align*}
    \Pr\!\rbr{S_n \ge R/8+r}
    &\le
    \exp\!\rbr{-\tfrac{r}{C\varepsilon}}
    \rbr{
        1+\tsum_{j\ge 1}\exp (-\tfrac{2^{j-1}}{C})
    } \\
    &\le
    C'\exp\!\rbr{-\tfrac{r}{C\varepsilon}},
\end{align*}
where $C'>0$ depends only on $c_0$. Taking
\begin{align*}
    r = C\varepsilon\log(2C'/\delta)
\end{align*}
gives
\begin{align*}
    \Pr\!\rbr{S_n \ge R/8+r} \le \delta/2.
\end{align*}
The same argument applied to $-S_n$ gives the lower tail, and a union bound yields
\begin{align*}
    \Pr\!\rbr{|S_n| \ge R/8+r} \le \delta.
\end{align*}
Finally, increasing the constant once more,
$r\le c\varepsilon\log(e/\delta)$.
\end{proof}

\section{Proof of Theorem~\ref{thm:adaptive-rate}}\label{app:adaptive-results}

This section proves the adaptive convergence results of Section~\ref{sec:results}. We first prove Theorem~\ref{thm:adaptive-rate} using Theorem~\ref{thm:main-theorem} and the estimation primitives, then prove the supporting Lemmas~\ref{lem:projection}, \ref{lem:NE-concentration}, and~\ref{lem:adaptive-tuning}.

\restatetheorem{thm:adaptive-rate}{
There exist absolute constants $C, c > 0$ such that, if $d \ge \log(2T^2)$, Algorithm~\ref{alg:main-algorithm} is run with $m_t = P_f^\smp(x_t,\cS_t)$ and step sizes $\alpha_t$ from \eqref{eq:adaptive-stepsize} with confidence parameter $\delta=1/(2T)$, $\gamma = C_0 \log(2T^2)$, and $N = \ceil{C\log(2T)}$, and the smoothing radius satisfies $\smp \le \frac{B/L}{\sqrt{T(d^3 + dk)\log(2T^2)}}$, then
\begin{align*}
    \tfrac{1}{L}\,\E\!\sbr{\tnorm{\nabla f(\widebar{x})}^2}
        \;\le\; c\,\E\!\sbr{\tfrac{(D_f^2 + B^2/L^2)\log(2T)}{\cA^*_T}},
\end{align*}
where $\cA^*_T = \tsum_{t=1}^T \alpha^*_t$ for $\alpha_t^*$ in \eqref{eq:oracle-stepsize}.
}

\begin{proof}
When $T=1$, the output is $x_1$ and the conclusion follows directly from $\tnorm{\nabla f(x_1)}^2\le B^2$, $\cA_1^*\le1/L$, and $c\ge1/\log2$. We therefore assume $T\ge2$.
Set $\delta := 1/T$ and $\iota := \log(2T/\delta) = \log(2T^2)$, so that
$\gamma = C_0\iota$.

\emph{Step 0: High-probability events.}
Choose the sample constant $C$ at least twice the corresponding constant in Lemma~\ref{lem:adaptive-tuning}; then $N=\ceil{C\log(2T)}$ meets its sample requirement at failure probability $1/(2T)$, since $\log(4T^2)=2\log(2T)$.
Let $\cE_1$ be the event from Lemma~\ref{lem:adaptive-tuning} applied with
failure probability $\delta/2=1/(2T)$, on which local optimality and balance control
hold simultaneously for all $t \in [T]$. By Lemma~\ref{lem:projection},
deterministically,
\begin{align*}
    \|m_t - P_{\cS_t}\nabla f(x_t)\|^2 
    \le \frac{kL^2\smp^2}{4}
    \qquad \text{for all } t \in [T].
\end{align*}
The assumed upper bound on $\smp$ also ensures
\begin{align*}
    \|m_t\|
    \le \|P_{\cS_t}\nabla f(x_t)\| + \frac{L\smp\sqrt{k}}{2}
    \le 2B,
\end{align*}
so the reference-vector condition in Theorem~\ref{thm:main-theorem} holds.
Let $\cE_2$ be the event from that theorem applied with failure probability
$\delta$; its logarithmic parameter is precisely $\iota$. Thus, for
$\cE := \cE_1 \cap \cE_2$, we have $\Pr(\cE) \ge 1-3\delta/2$.

\emph{Step 1: Controlling $\cB_T^\gamma$ on $\cE_1$.}
Because $m_t-P_{\cS_t}\nabla f(x_t) \in \cS_t$ and
$P_{\cS_t}^{\perp}\nabla f(x_t) \in \cS_t^\perp$, for each $t \in [T]$,
\begin{align*}
    \|m_t - \nabla f(x_t)\|^2
    &= \|m_t - P_{\cS_t}\nabla f(x_t)\|^2
    + \|P_{\cS_t}^{\perp}\nabla f(x_t)\|^2\\
    &\le \frac{kL^2\smp^2}{4}
    + \|P_{\cS_t}^{\perp}\nabla f(x_t)\|^2.
\end{align*}
On $\cE_1$, substituting this bound and applying the balance control from
Lemma~\ref{lem:adaptive-tuning} with failure probability $\delta/2$ gives
\begin{align*}
    \alpha_t^2 d\,\|m_t - \nabla f(x_t)\|^2
    &\le \alpha_t^2 d\,\|P_{\cS_t}^{\perp}\nabla f(x_t)\|^2
    + \frac{\alpha_t^2 d kL^2\smp^2}{4}\\
    &\le \frac{\alpha_t}{2\gamma L}\|\nabla f(x_t)\|^2
    + c\,\smp^2\!\left(d^3 + d\log^2(4NT/\delta) + dk\right).
\end{align*}
Here we also used $\alpha_t \le 1/(\gamma L)$. Since
$N=O(\log T)$ and $d\ge\iota$, we have
$\log(4NT/\delta)\le c d$. Dropping the remaining negative descent credit
from each summand and summing over $t$ therefore yields
\begin{align}\label{eq:BT-on-E-v3}
    \cB_T^\gamma
    &\le c\,\smp^2 T(d^3 + dk).
\end{align}

\emph{Step 2: Applying Theorem~\ref{thm:main-theorem} on $\cE$.}
Deterministically, the step sizes satisfy
$\alpha_t \le 1/(\gamma L) = 1/(C_0\iota L)$, which is precisely the 
hypothesis of Theorem~\ref{thm:main-theorem}. On $\cE_2$, the
high-probability bound gives
\begin{align*}
    \tfrac{1}{L}\tsum_{t=1}^T \alpha_t\|\nabla f(x_t)\|^2
    \le C\rbr{D_f^2 + \tfrac{B^2}{L^2} 
    + \smp^2 Td^3 + \cB_T^\gamma\,\iota}.
\end{align*}
Substituting~\eqref{eq:BT-on-E-v3} and using $\iota=\log(2T^2)$,
\begin{align}\label{eq:sum-on-E-v3}
    \tfrac{1}{L}\tsum_{t=1}^T \alpha_t\|\nabla f(x_t)\|^2
    &\le c\rbr{D_f^2 + \tfrac{B^2}{L^2} 
    + \smp^2 T (d^3 + dk)\log(2T^2)}.
\end{align}

\emph{Step 3: Taking expectations.}
On $\cE$ (probability $\ge 1-3/(2T)$), local optimality gives
$\alpha_t \ge \alpha_t^*/(c_0\gamma)$ for all $t$, so 
$\cA_T \ge \cA_T^*/(\gamma c_0)$. Dividing~\eqref{eq:sum-on-E-v3} by $\cA_T$ 
and using the output rule,
\begin{align*}
    \ind_{\cE}\cdot\E_R\!\sbr{\|\nabla f(\widebar x)\|^2}
    \le \tfrac{c\gamma L}{\cA_T^*}
    \rbr{D_f^2 + \tfrac{B^2}{L^2} 
    + \smp^2 T (d^3 + dk)\log(2T^2)}.
\end{align*}
On $\cE^c$ (probability $\le 3/(2T)$), we bound
$\|\nabla f(\widebar x)\|^2 \le B^2$. Taking full expectations,
\begin{align*}
    \tfrac{1}{L}\,\E\!\sbr{\|\nabla f(\widebar x)\|^2}
    &\le c\,\gamma\,\E\!\sbr{\tfrac{D_f^2 + \tfrac{B^2}{L^2} 
    + \smp^2 T (d^3 + dk)\log(2T^2)}{\cA_T^*}}
    + \tfrac{3B^2}{2LT}.
\end{align*}
The assumed bound on $\smp$ makes the smoothing term at most $B^2/L^2$.
Finally, $\gamma=O(\log(2T))$ and $\cA_T^*\le T/L$, so the failure-event
term is absorbed into the first term, proving the claim.
\end{proof}

\subsection{Proof of Lemma~\ref{lem:projection}}
\restatelemma{lem:projection}{
For every $x \in \R^d$ and $\cS \in \Gr_k(\R^d)$, Algorithm~\ref{alg:projection} satisfies
\begin{align*}
    \|P^{\smp}_{f}(x, \cS) - P_{\cS}\nabla f(x)\|^2
    \le \frac{kL^2\smp^2}{4}.
\end{align*}
}

\begin{proof}
Let $B=[b_1,\ldots,b_k]$ be the orthonormal basis matrix used by
Algorithm~\ref{alg:projection}, and let $y:=B^\top\nabla f(x)$. For each
$j\in[k]$, $L$-smoothness gives
\begin{align*}
    \left|\hty_j-\pair{\nabla f(x),b_j}\right|
    &= \frac{1}{\smp}\left|f(x+\smp b_j)-f(x)
       -\smp\pair{\nabla f(x),b_j}\right| \le \frac{L\smp}{2},
\end{align*}
where we used $\|b_j\|=1$. Since $P_{\cS}=BB^\top$ and $B^\top B=I_k$,
\begin{align*}
    \|P_f^\smp(x,\cS)-P_{\cS}\nabla f(x)\|^2
    &= \|B(\hty-y)\|^2
     = \sum_{j=1}^k\left|\hty_j-\pair{\nabla f(x),b_j}\right|^2 \le \frac{kL^2\smp^2}{4}.
\end{align*}
\end{proof}

\subsection{Proofs of Lemmas~\ref{lem:NE-concentration} and~\ref{lem:adaptive-tuning}}

We first establish the norm estimation guarantee (Lemma~\ref{lem:NE-concentration}), then use it to prove Lemma~\ref{lem:adaptive-tuning}.

\restatelemma{lem:NE-concentration}{
    There exist absolute constants $C, c > 0$ such that for every $x \in \R^d$, linear subspace $\cS$ of $\R^d$, $\smp > 0$, $\delta \in (0,1)$, and $N \ge C \log(2/\delta)$, it holds with probability at least $1-\delta$ that
    \begin{align*}
        N^{\smp}_f(x, \cS, N) \in \sbr{\tfrac{1}{2}\|P_{\cS} \nabla f(x)\|^2 - c\varepsilon,\, \tfrac{3}2\|P_{\cS} \nabla f(x)\|^2 + c\varepsilon},
    \end{align*}
    where $\varepsilon = \smp^2 L^2 (d^2 + \log^2(2N/\delta))$.
}

\begin{proof}
    Let $v = P_{\cS} \nabla f(x)$. By $L$-smoothness of $f$,
    \[
        \hty_i = \pair{\nabla f(x), P_{\cS} u_i} + \smp \epsilon_i = \pair{v, u_i} + \smp \epsilon_i,
    \]
    with $|\epsilon_i| \le \tfrac{L}{2}\|P_{\cS} u_i\|^2 \le \tfrac{L}{2}\|u_i\|^2$. Decompose the estimator as follows
    \begin{align*}
    \frac{1}{N}\sum_{i=1}^N \hty_i^2 
    &\;=\; \frac{1}{N}\sum_{i=1}^N \pair{v, u_i}^2 \;+\; \frac{2 \smp}{N}\sum_{i=1}^N \pair{v, u_i}\,\epsilon_i \;+\; \frac{\smp^2}{N}\sum_{i=1}^N \epsilon_i^2.
    \end{align*}

    Let $S_1 = \frac{1}{N}\sum_{i=1}^N \pair{v, u_i}^2$ and $S_2 = \frac{1}{N}\sum_{i=1}^N \tnorm{u_i}^4$. Using the fact that $\epsilon_i^2 \le \frac{L^2}{4} \tnorm{u_i}^4$,
    \begin{align}\label{eq:decomp-projection}
        \abr{\frac{1}{N}\sum_{i=1}^N \hty_i^2 - S_1} \le \frac{S_1}{4} + \frac{5\smp^2}{N} \sum_{i=1}^N \epsilon_i^2 \le \frac{S_1}{4} + \frac{5\smp^2 L^2 S_2}{4}.
    \end{align}

    \textbf{Step 1: Concentrating $S_1$ around $\|v\|^2$.}\quad When $v = 0$, $S_1 = 0$ holds almost surely. When $v \ne 0$,  $\pair{v, u_i}^2 / \|v\|^2 \iid \chi^2_1$. By Laurent-Massart (Lemma~\ref{lem:laurent-massart}), for $N \ge 256 \log(4/\delta)$,
    \begin{equation*}
         \P\rbr{|S_1 - \|v\|^2| \le \tfrac{1}{5}\|v\|^2} \ge 1 - \delta/2.
    \end{equation*}
    
    \textbf{Step 2: Bounding $S_2$.}\quad
    Since $\tnorm{u_i}^2 \iid \chi_d^2$, by the union bound of Laurent-Massart (Lemma~\ref{lem:laurent-massart}),
    \begin{align*}
        \Pr\rbr{\max_{i\in[N]} |\tnorm{u_i}^2 - d| \le 2\sqrt{d \log(4N/\delta)} + 2\log(4N/\delta)} \ge 1 - \delta/2. 
    \end{align*}
    Consequently, as $S_2 \le \max_{i\in[N]} \tnorm{u_i}^4 = (\max_{i\in[N]} \tnorm{u_i}^2)^2$, we have
    \begin{align*}
        \Pr\rbr{S_2 \le 18 (d^2 + \log^2(4N/\delta))} \ge 1 - \delta/2. 
    \end{align*}
    
    \textbf{Conclusion.}\quad 
    Combining \eqref{eq:decomp-projection} with a union-bound over events in Steps 1 and 2, we have that with probability at least $1-\delta$, 
    \begin{align*}
        \frac{1}{N}\sum_{i=1}^N \hty_i^2 
        &\in \sbr{\frac{3S_1}{4} - \frac{5\smp^2 L^2 S_2}{4} , \frac{5S_1}{4} + \frac{5\smp^2 L^2 S_2}{4}}\\
        &\subset \sbr{\frac{\tnorm{v}^2}{2} - 25 \smp^2 L^2 (d^2 + \log^2(4N/\delta)), \frac{3\tnorm{v}^2}{2} +  25 \smp^2 L^2 (d^2 + \log^2(4N/\delta))}.
    \end{align*}
    Here $(3/4)(4/5)=3/5\ge1/2$, $(5/4)(6/5)=3/2$, and $(5/4)\cdot18\le25$. Since $\log(4/\delta)\le2\log(2/\delta)$ and $\log(4N/\delta)\le2\log(2N/\delta)$, the stated constants may be taken as $C=512$ and $c=100$. This proves the lemma.
\end{proof}

\restatelemma{lem:adaptive-tuning}{
    There exist absolute constants $C, c_0, c_1 > 0$ such that the following holds. For any $\delta\in(0,1)$ and $N \ge C\log(2T/\delta)$, set $\varepsilon := \smp^2 L^2(d^2 + \log^2(2NT/\delta))$ and define
\begin{align*}
    \alpha_t
    :=
    \tfrac{1}{\gamma\,L}\,\min\cbr{1,\,\tfrac{N^{\smp}_f(x_t,\R^d,N) + c_1\varepsilon}{16d\,\max\{0,\, N^{\smp}_f(x_t,\cS_t^{\perp},N) - c_1\varepsilon\}}},
\end{align*}
If the denominator is zero, define $\alpha_t:=1/(\gamma L)$.
Then with probability at least $1-\delta$, for all $t \in [T]$ simultaneously:
    \begin{enumerate}[leftmargin=0.75cm,noitemsep]
        \item (Local Optimality) $\quad \alpha_t \ge \tfrac{\alpha_t^*}{c_0\gamma}$;
        \item (Balance Control) $\quad \alpha_t^2 d\tnorm{P_{\cS_t}^{\perp} \nabla f(x_t)}^2 \le \tfrac{\alpha_t}{2\gamma L}\tnorm{\nabla f(x_t)}^2 + c_0\,\smp^2 (d^3 + d\log^2(2NT/\delta))$.
    \end{enumerate}
}
\begin{proof}
    Let $c_{\mathrm{NE}}$ be the error constant in Lemma~\ref{lem:NE-concentration}, take $c_1=4c_{\mathrm{NE}}$, and abbreviate $c:=c_1$ within this proof.
    Write $a_t := N^{\smp}_f(x_t,\R^d,N)$ and $b_t := N^{\smp}_f(x_t,\cS_t^{\perp},N)$, which target $\tnorm{\nabla f(x_t)}^2$ and $\tnorm{P^{\perp}_{\cS_t}\nabla f(x_t)}^2$, respectively. Apply Lemma~\ref{lem:NE-concentration} conditionally on the information available before each call, with failure probability $\delta/(2T)$. Its error radius is at most $4\varepsilon$, since $\log(4NT/\delta)\le2\log(2NT/\delta)$. A union bound over the $2T$ calls therefore gives, for $N \ge C\log(2T/\delta)$ with $C$ sufficiently large, with probability at least $1-\delta$,
    \begin{align*}
        a_t \in \sbr{\tfrac{1}{2}\tnorm{\nabla f(x_t)}^2 - c\varepsilon,\, \tfrac{3}{2}\tnorm{\nabla f(x_t)}^2 + c\varepsilon}
    \end{align*}
    and similarly for $b_t$ targeting $\tnorm{P^\perp_{\cS_t}\nabla f(x_t)}^2$, simultaneously for all $t \in [T]$. We work on this event throughout the rest of the proof.

    \emph{Local optimality.} The bounds $a_t + c\varepsilon \ge \tfrac{1}{2}\tnorm{\nabla f(x_t)}^2$ and $\max\{0, b_t - c\varepsilon\} \le \tfrac{3}{2}\tnorm{P^{\perp}_{\cS_t}\nabla f(x_t)}^2$ yield
    \begin{align*}
        \alpha_t \ge \tfrac{1}{\gamma L}\,\min\cbr{1,\, \tfrac{\tnorm{\nabla f(x_t)}^2}{64d\,\tnorm{P^{\perp}_{\cS_t}\nabla f(x_t)}^2}} \ge \tfrac{\alpha_t^*}{64\gamma}.
    \end{align*}
    (When $b_t - c\varepsilon \le 0$ or $\nabla f(x_t) = 0$, $\alpha_t = 1/(\gamma L) \ge \alpha_t^*/\gamma$ directly.)

    \emph{Balance Control.} We bound $\alpha_t^2 d\,\tnorm{P^{\perp}_{\cS_t}\nabla f(x_t)}^2$ in two cases.

    \textbf{Case 1: $c\varepsilon \le \tfrac{1}{8}\tnorm{P^{\perp}_{\cS_t}\nabla f(x_t)}^2$.} Then $b_t - c\varepsilon \ge \tfrac{1}{4}\tnorm{P^{\perp}_{\cS_t}\nabla f(x_t)}^2$ and $a_t + c\varepsilon \le \tfrac{3}{2}\tnorm{\nabla f(x_t)}^2 + 2c\varepsilon$, which guarantees
    \begin{align*}
        \alpha_t \le \tfrac{1}{\gamma L}\,\min\cbr{1,\, \tfrac{\tnorm{\nabla f(x_t)}^2}{2d\,\tnorm{P^{\perp}_{\cS_t}\nabla f(x_t)}^2}} + \tfrac{c\varepsilon}{2\gamma Ld\,\tnorm{P^{\perp}_{\cS_t}\nabla f(x_t)}^2}.
    \end{align*}
    Multiplying by $\alpha_t d\,\tnorm{P^{\perp}_{\cS_t}\nabla f(x_t)}^2$ and using $\alpha_t \le 1/(\gamma L)$,
    \begin{align*}
        \alpha_t^2 d\,\tnorm{P^{\perp}_{\cS_t}\nabla f(x_t)}^2
        \le \tfrac{\alpha_t}{2\gamma L}\tnorm{\nabla f(x_t)}^2 + \tfrac{c\varepsilon}{2 L^2}
        \le \tfrac{\alpha_t}{2\gamma L}\tnorm{\nabla f(x_t)}^2 + c\,\smp^2(d^2 + \log^2(2NT/\delta)).
    \end{align*}

    \textbf{Case 2: $c\varepsilon > \tfrac{1}{8}\tnorm{P^{\perp}_{\cS_t}\nabla f(x_t)}^2$.} Then $\tnorm{P^{\perp}_{\cS_t}\nabla f(x_t)}^2 < 8c\varepsilon$ and, using $\alpha_t \le 1/(\gamma L) \le 1/L$,
    \begin{align*}
        \alpha_t^2 d\,\tnorm{P^{\perp}_{\cS_t}\nabla f(x_t)}^2 \le \tfrac{1}{\gamma^2 L^2}\cdot d \cdot 8c\varepsilon \le 8c\,\smp^2(d^3 + d\log^2(2NT/\delta)).
    \end{align*}

    Combining both cases and taking $c_0 = \max\{64,8c\}$ proves both claims of the lemma.
\end{proof}

\section{Technical preliminaries}

\begin{fact}\label{fact:gaussian-outer-product}
    Let $u\sim \cN(0, I_d)$. For $k \in \N$, it holds that
    \begin{align*}
    \E\!\sbr{(u u^\top)^k} = \prod_{j=1}^{k-1}(d+2j)\, I_d.
    \end{align*}
\end{fact}

\begin{fact}\label{fact:gaussian-power-norm}
    Let $u \sim \cN(0, I_d)$. For $k \ge 0$, it holds that
    \begin{align*}
        \E\!\sbr{\tnorm{u}^k} = 2^{k/2} \, \frac{\Gamma(\frac{d+k}{2})}{\Gamma(\frac{d}{2})} \le (d+k)^{k/2}.
    \end{align*}
    In particular, for $m \in \N$, $\E\!\sbr{\tnorm{u}^{2m}} = d (d+2) \ldots (d+2m-2)$.
\end{fact}

\begin{lemma}[Laurent--Massart]\label{lem:laurent-massart}
    Let $X_1,\ldots,X_N \iid \chi^2_d$ and $\bar X = \frac{1}{N}\sum_{i=1}^N X_i$. For every $\delta \in (0,1)$,
    \begin{align*}
        \P\!\rbr{\abr{\bar X-d} \le 2\sqrt{\frac{d\log(2/\delta)}{N}} + \frac{2\log(2/\delta)}{N}} \ge 1-\delta.
    \end{align*}
\end{lemma}

\begin{lemma}[Isserlis]\label{lem:Isserlis}
    For fixed $a \in \R^d$ and $\Sigma \in \R^{d\times d}$,
    \begin{align*}
        \E_{z\sim\cN(0,I_d)}\bigl[(a^\top z)^2 (z^\top \Sigma z)\bigr] = \|a\|^2\,\tr(\Sigma) + 2\,a^\top \Sigma\,a.
    \end{align*}
\end{lemma}
\begin{proof}
    Write $(a^\top z)^2(z^\top \Sigma z) = \sum_{i,j,k,l} a_i a_j \Sigma_{kl}\, z_i z_j z_k z_l$.
    Since $z \sim \cN(0,I_d)$, Isserlis' theorem gives
    $\E[z_i z_j z_k z_l] = \delta_{ij}\delta_{kl} + \delta_{ik}\delta_{jl} + \delta_{il}\delta_{jk}$.
    Substituting and summing over each pairing:
    \begin{align*}
        \sum_{i,j,k,l} a_i a_j \Sigma_{kl}\,\delta_{ij}\delta_{kl}
        &= \Bigl(\sum_i a_i^2\Bigr)\Bigl(\sum_k \Sigma_{kk}\Bigr)
        = \|a\|^2\,\tr(\Sigma), \\
        \sum_{i,j,k,l} a_i a_j \Sigma_{kl}\,\delta_{ik}\delta_{jl}
        &= \sum_{i,j} a_i a_j \Sigma_{ij}
        = a^\top \Sigma\, a, \\
        \sum_{i,j,k,l} a_i a_j \Sigma_{kl}\,\delta_{il}\delta_{jk}
        &= \sum_{i,j} a_i a_j \Sigma_{ji}
        = a^\top \Sigma\, a,
    \end{align*}
    where the last equality uses $a^\top\Sigma^\top a=a^\top\Sigma a$. Summing the three terms yields the result.
\end{proof}

\end{document}